\documentclass[11pt]{article}

\usepackage[final]{acl}

\usepackage{times}
\usepackage{latexsym}
\usepackage{amsmath, amssymb, amsthm}
\usepackage{mathtools}
\usepackage{bm}
\usepackage{enumitem}
\usepackage{hyperref}
\usepackage{todonotes}
\usepackage{tikz}
\usepackage{subcaption}

\usepackage[T1]{fontenc}
\usepackage[utf8]{inputenc}

\usepackage{microtype}

\usepackage{inconsolata}

\usepackage{graphicx}
\usepackage{float}
\usepackage{booktabs}

\usepackage[framemethod=TikZ]{mdframed}
\definecolor{promptbg}{rgb}{0.945,0.957,0.980}
\definecolor{promptframe}{rgb}{0.70,0.74,0.82}
\surroundwithmdframed[
  backgroundcolor=promptbg,
  linecolor=promptframe,
  linewidth=0.5pt,
  roundcorner=3pt,
  innertopmargin=5pt,
  innerbottommargin=5pt,
  innerleftmargin=5pt,
  innerrightmargin=5pt,
  skipabove=5pt,
  skipbelow=3pt,
]{verbatim}
\usepackage{etoolbox}
\AtBeginEnvironment{verbatim}{\scriptsize}

\newtheorem{theorem}{Theorem}
\newtheorem{proposition}[theorem]{Proposition}
\newtheorem{lemma}[theorem]{Lemma}
\newtheorem{corollary}[theorem]{Corollary}

\theoremstyle{definition}

\newtheorem{assumption}{Assumption}
\newtheorem{remark}{Remark}

\DeclareMathOperator*{\E}{\mathbb{E}}
\DeclareMathOperator{\Prob}{\mathbb{P}}
\newcommand{\dtv}{d_{\mathrm{TV}}}
\newcommand{\dks}{d_{\mathrm{KS}}}
\newcommand{\R}{\mathbb{R}}
\newcommand{\X}{\mathcal{X}}
\newcommand{\Y}{\mathcal{Y}}
\newcommand{\Z}{\mathcal{Z}}
\newcommand{\U}{\mathcal{U}}
\newcommand{\Wone}{W_1}
\newcommand{\PredSet}{\Gamma}

\newcommand{\ind}{\mathbf{1}}

\newcommand{\Dcal}{\mathcal{D}_{\mathrm{cal}}}

\title{Conformalized Large Language Models under Configuration Shift}

\author{%
	Yuqicheng Zhu\textsuperscript{1,2,7}\thanks{This work was conducted during a research visit to the University of Oxford.}, 
	Jialin Yu\textsuperscript{3},
	Lin Li\textsuperscript{3},
	Gengyuan Zhang\textsuperscript{4}\footnotemark[1],
	Zhen Yang\textsuperscript{5},\\
	\textbf{Steffen Staab\textsuperscript{1,6},} 
	\textbf{Puneet Dokania\textsuperscript{2,3},}
	\textbf{Philip Torr\textsuperscript{3},}
	\textbf{Jie Tang\textsuperscript{5},}
	\textbf{Evgeny Kharlamov\textsuperscript{2,7}} 
	\\
	\textsuperscript{1}University of Stuttgart, 
	\textsuperscript{2}Robert Bosch GmbH,
	\textsuperscript{3}University of Oxford,
	\textsuperscript{4}LMU Munich,\\
	\textsuperscript{5}Tsinghua University,
	\textsuperscript{6}University of Southampton,
	\textsuperscript{7}University of Oslo
	\\
	\texttt{yuqicheng.zhu@de.bosch.com}
}

\begin{document}
\maketitle

\begin{abstract}
Conformal prediction (CP) is a distribution-free framework for uncertainty quantification that has recently been adapted to large language models (LLMs), providing prediction sets with finite-sample coverage guarantees under exchangeability.
Yet for LLMs, nonconformity scores are often induced by an inference pipeline, not just a fixed model, making them depend not only on the data distribution but also on configurable factors such as the prompt template, decoding parameters, and deployment setting.
Since such configurations are routinely modified in practice but rarely treated as a source of shift, their impact on CP validity remains poorly understood.
We call this \emph{configuration shift} and study it
systematically along three axes: prompt template, decoding
temperature, and weight quantization. 
In a broad empirical study spanning $9$ LLMs, $4$ datasets, and $4$ nonconformity scores, we find that configuration shift consistently erodes CP validity, often driving empirical coverage below the target. By contrast, efficiency is largely preserved: valid prediction sets remain close in size to the i.i.d. baseline.
We derive coverage lower bounds that attribute this loss to a discrepancy between calibration and test score distributions, and use their finite-sample plug-in versions as empirical diagnostics of shift severity.
We further show that these findings lead to practical mitigations: bound-inspired recalibration is effective with limited test examples, while fragility-aware calibration ensembling recovers much of the lost coverage without test data.
\end{abstract}

\section{Introduction}
As large language models (LLMs) are increasingly deployed in high-stakes domains, including medicine, law, education, and scientific decision support, reliable uncertainty quantification has become critical \citep{shorinwa2024uqsurvey}.
Conformal prediction (CP) offers a distribution-free framework for quantifying uncertainty with statistical guarantees \citep{vovk2005,angelopoulos2021gentle}.
Given a held-out calibration set and a model-specific uncertainty score (nonconformity score), CP constructs a prediction set that, under exchangeability, contains the true answer with user-specified probability \citep{papadopoulos2002,lei2018}.
The size of the resulting set reflects the model's effective uncertainty at the chosen confidence level.

The statistical guarantees of CP rely on exchangeability of the nonconformity scores.
In classical supervised learning, the predictor and scoring rule are fixed after training; for example, a score may be one minus the softmax probability assigned to the true label \citep{sadinle2019}. Under this fixed score map, exchangeable examples induce exchangeable scores.
For LLM systems, however, nonconformity scores are often produced by an inference pipeline \citep{kumar2023mcq,su2024api,quach2024clm}. The pipeline may include a prompt template, decoding configuration, and model deployment. Changing any of these components can change the score assigned to the same labelled example \citep{sclar2024quantifying, shi-etal-2024-thorough, helcig2026statistically}. We call this \emph{configuration shift}. 
Because CP validity depends on the resulting scores, such shifts can break score exchangeability and invalidate the guarantees.

Most work on CP under distribution shift models changes in the example distribution, such as covariate, domain, subject, or temporal shift \citep{tibshirani2019covariate,gibbs2021adaptive,barber2023,cauchois2024robust}. Configuration shift is orthogonal: the example distribution is unchanged, while the map from examples to scores changes.
Some studies include prompt-sensitivity or robustness checks \citep{kumar2023mcq,su2024api}, but these are usually auxiliary empirical tests over a small number of variants. They do not systematically study pipeline changes as a source of score-level distribution shift.

Yet such pipeline changes are routine in deployed LLM systems.
Prompt templates, decoding parameters, and quantization levels vary across users, tasks, and hardware constraints.
Thus, configuration shift is both underexplored by existing literature and difficult to avoid in practice. We therefore study it as a first-class threat to conformal validity in LLM pipelines, with three contributions.

\textbf{First}, 
we conduct a systematic empirical study of configuration shift in conformal LLM pipelines.
We sweep three deployment-relevant axes (prompt template,
decoding temperature, and weight quantization) across $9$ LLMs, $4$ datasets, and $4$ nonconformity scores. We find that the shift
substantially erodes CP validity: empirical coverage drops well below
target and gathering more calibration data deepens the bias. 
The damage falls mostly on validity rather than efficiency, i.e., among the runs that remain valid, set sizes stay close to their i.i.d. values.
\textbf{Second}, 
we explain this failure through coverage lower bounds that quantify how discrepancies between the calibration and test score distributions reduce guaranteed coverage. We further derive label-free variants that depend only on the model's predictions, yielding plug-in diagnostics that can be computed without test labels.
\textbf{Third}, 
we show that the empirical and theoretical results are
directly actionable. The discrepancy term in lower bound motivates threshold corrections, while the empirical fragility profile motivates calibration ensembling schemes that deliberately over-weight configurations prone to undercoverage. 

\section{Preliminaries}
\label{sec:prelim}

\subsection{Conformal Prediction}
\label{sec:cp-background}
CP provides a principled framework for uncertainty quantification by constructing a \emph{prediction set}, namely a set of candidate solutions for a given task, with a guaranteed probability of containing the ground-truth solution at a predefined confidence level \cite{vovk2005}.
It assigns a (nonconformity) score to each candidate solution and uses a calibrated threshold to include only those candidates required to satisfy the desired confidence level.
The size of the resulting prediction set can therefore be interpreted as a measure of uncertainty: larger sets indicate greater uncertainty in the prediction.

\paragraph{Nonconformity Score.} 
Let $\X$ denote the input space,
$\Y$ the label space, and $Z=(X,Y)$ take values in $\Z:=\X\times\Y$.
A nonconformity score is a measurable map
\[
s:\X\times\Y\to\R,
\]
which measures how poorly a candidate label $y$ fits input $x$
according to the underlying model.

\paragraph{Calibration.}
Given held-out calibration data $\Dcal=\{(X_i,Y_i)\}_{i=1}^n$,
compute the calibration scores $S_i=s(X_i,Y_i)$ for each example.
For a target miscoverage level $\alpha\in(0,1)$, set
$k_\alpha:=\lceil(n+1)(1-\alpha)\rceil$ and take the threshold
\begin{equation}\label{eq:lambda-hat}
	\hat\lambda := S_{(k_\alpha)},
\end{equation}
the $k_\alpha$-th smallest value among $S_1,\ldots,S_n$; if
$k_\alpha=n+1$ (which occurs when $\alpha<1/(n+1)$) we adopt the
convention $\hat\lambda:=+\infty$.

\paragraph{Set Construction.}
For a new input $x$, the prediction set collects every label whose
score does not exceed the calibrated threshold:
\begin{equation}\label{eq:pred-set}
	\PredSet_{\hat\lambda}(x)
	:=\{y\in\Y: s(x,y)\le\hat\lambda\}.
\end{equation}

\begin{theorem}[\citealp{vovk2005,papadopoulos2002,lei2018}]
	\label{thm:split-cp}
	Let $S_1,\ldots,S_n$ be the calibration scores and $S_{n+1}=s(X_{n+1},Y_{n+1})$
	the score of a test point. If $(S_1,\ldots,S_n,S_{n+1})$ is
	exchangeable, then
	\begin{equation}\label{eq:split-cp-coverage}
		\Prob\!\bigl(Y_{n+1}\in\PredSet_{\hat\lambda}(X_{n+1})\bigr)
		\;\ge\; 1-\alpha.
	\end{equation}
\end{theorem}

\paragraph{Validity and efficiency.}
CP performance is assessed along two
complementary axes. 
Given a test set $\{(X_j,Y_j)\}_{j=1}^m$, \emph{validity} is measured by the empirical coverage
\begin{equation}\label{eq:empirical-coverage}
    \widehat C
    =
    \frac{1}{m}\sum_{j=1}^{m}
    \ind\!\bigl(Y_j\in\PredSet_{\hat\lambda}(X_j)\bigr).
\end{equation}
This estimates the coverage probability in
Equation~\eqref{eq:split-cp-coverage}.
\emph{Efficiency} is measured by the average prediction-set size
\begin{equation}\label{eq:average-set-size}
    \widehat L
    =
    \frac{1}{m}\sum_{j=1}^m |\PredSet_{\hat\lambda}(X_j)|.
\end{equation}
For methods with comparable validity, smaller $\widehat L$ indicates
more informative prediction sets.

\subsection{Conformal prediction for LLMs}
\label{sec:cp-llm}
Recently, CP has been applied to LLMs for tasks such as question answering, open-ended
generation, and summarization
\citep{kumar2023mcq,su2024api,quach2024clm}. Unlike the standard
setting, the nonconformity score is derived from the LLM, and
therefore depends not only on the data $(X,Y)$ but also on factors
governing LLM inference such as the prompt template, decoding
parameters or model deployment. We collect these
factors in an auxiliary variable $\xi\in\Xi$ and write the score as
\[
S = V(Z,\xi).
\]
The classical notation $s(X,Y)$ corresponds to the special case in
which $\xi$ is fixed.

For example, when logits are accessible and $\Y$ is finite, the LLM
induces a predictive distribution $\hat p_\xi(\cdot\mid x)$ over
candidate labels. Two common scores are LAC \citep{sadinle2019},
\[
V((x,y),\xi)=1-\hat p_\xi(y\mid x),
\]
and the logit margin \citep{vovk2005},
\[
\begin{aligned}
V((x,y),\xi)
&=\max_{k\neq y}\{\log \hat p_\xi(k\mid x)\\
&\qquad-\log \hat p_\xi(y\mid x)\}.
\end{aligned}
\]
Both scores become larger when the candidate label is less supported by the
model under configuration $\xi$. When logits are unavailable or when $\Y$ is not a fixed finite candidate set as in open-ended question answering, scores
can instead be built from sampled generations; we give the black-box
scores used in this paper in Appendix~\ref{app:blackbox-scores}.

\section{Problem Formulation}
The coverage guarantee in Theorem~\ref{thm:split-cp} requires
exchangeability of $(S_1,\ldots,S_n,S_{n+1})$. In classical CP,
this holds whenever the data are i.i.d.\ and the score function is
fixed. In the LLM setting, however, each score
$S=V(Z,\xi)$ depends on the inference configuration $\xi$, and
exchangeability holds only when calibration and test scores are
produced under the same $\xi$. 

In practice, a practitioner may
modify the prompt template, adjust decoding parameters, or switch to a quantised model variant. The data distribution
$P_Z$ remains unchanged, yet the score distribution shifts:
$P_S^{\xi_0}\neq P_S^{\xi'}$. We refer to this as
\emph{configuration shift}, to distinguish it from the covariate or
label shifts studied in the broader CP literature \citep{tibshirani2019covariate,gibbs2021adaptive,barber2023,cauchois2024robust}, where the data
distribution $P_Z$ itself changes.

Let $\xi_0\in\Xi$ denote the \emph{reference configuration}
under which the calibration set
$\Dcal=\{(X_i,Y_i)\}_{i=1}^n$ is collected, yielding calibration
scores
\[
S_i = V(Z_i,\xi_0), \qquad i=1,\ldots,n.
\]
At test time the configuration is perturbed to some
$\xi'\in\Xi$, so that for a test point
$Z_{n+1}=(X_{n+1},Y_{n+1})$ the test score is
\[
S_{n+1} = V(Z_{n+1},\xi').
\]
Even if $Z_1,\ldots,Z_{n+1}$ are i.i.d.\ from $P_Z$, the
joint sequence $(S_1,\ldots,S_n,S_{n+1})$ can fail to be
exchangeable when $\xi'\neq\xi_0$, and the guarantee
\eqref{eq:split-cp-coverage} no longer applies.

\section{Empirical Evidence of Configuration Sensitivity}
\label{sec:empirical}

We empirically test whether the configuration shift defined in
Section~\ref{sec:cp-llm} produces measurable failures of CP validity
and efficiency.

\subsection{Experimental Setting}
\label{sec:exp-setting}

\paragraph{Models.}
We evaluate nine open-weight LLMs spanning three
families:
Llama-3.2-Instruct 1B/3B and Meta-Llama-3-8B-Instruct \citep{meta2024llama32,grattafiori2024llama3}; Gemma-3-it 1B/4B/12B \citep{gemmateam2025gemma3}; and Qwen3.5 0.8B/4B/9B \citep{qwen3.5}. 

\paragraph{Datasets.}
We use four widely adopted benchmarks covering two task settings.
For multiple-choice question answering (MCQA) with a fixed
four-option label space we use \emph{MMLU}
\citep{hendrycks2021mmlu} and \emph{MedMCQA}
\citep{pal2022medmcqa}. For open-ended question answering (OEQA)
with an unbounded answer space we use \emph{Natural Questions (NQ)}
\citep{kwiatkowski2019nq} and \emph{TriviaQA}
\citep{joshi2017triviaqa}.

\paragraph{Nonconformity scores.}
We instantiate the two white-box scores introduced in
Section~\ref{sec:cp-llm} (LAC and logit margin; more details in
Appendix~\ref{app:whitebox-scores}), together with the
two black-box scores defined in Appendix~\ref{app:blackbox-scores}
(self-consistency and LoFreeCP).

\paragraph{Configuration shifts.}
We study three configuration axes that practitioners vary
in deployed LLM pipelines while keeping the data distribution $P_Z$
unchanged:
\begin{enumerate}[leftmargin=*,nosep]
    \item \emph{Prompt template shift.} For each task type we sweep
    over a suite of prompt templates that preserve task semantics but
    vary the inference prompt along three axes: (i)~\emph{surface
    format} of the answer options (MCQA) or of the question/answer
    prefix tokens (OEQA); (ii)~\emph{instruction syntax and wording},
    spanning persona, verbosity, and output schema; and
    (iii)~\emph{content ablation}, in which exactly one scaffolding
    component (subject preamble, chain-of-thought directive, or
    few-shot demonstration) is removed. The MCQA suite (MMLU,
    MedMCQA) contains $14$ templates anchored at a paper-faithful
    one-shot Baseline~\citep{kumar2023mcq}; the OEQA suite (NQ,
    TriviaQA) contains $7$ templates anchored at a LoFreeCP-style
    five-shot Baseline~\citep{su2024api}. Concrete templates are documented in
    Appendix~\ref{app:prompt-templates}.
    \item \emph{Temperature shift.} Five LLM decoding temperatures
    $T\in\{0.5, 0.7, 1.0, 1.3, 1.5\}$, with $T=1.0$ as the reference.
    \item \emph{Quantization shift.} Four GGUF precisions
    (\texttt{Q8}, \texttt{Q6}, \texttt{Q4}, and \texttt{Q2})
    for three instruction-tuned models: Qwen3.5-9B,
    Meta-Llama-3-8B-Instruct, and Gemma-3-12B-it
    \citep{bartowski_qwen35_9b_gguf,
    	quantfactory_llama3_8b_instruct_gguf,
    	unsloth_gemma3_12b_it_gguf}.
\end{enumerate}

\paragraph{Protocol.}
Calibration is performed under the reference configuration $\xi_0$, while
testing is performed under a target configuration $\xi'$, including the
in-distribution case $\xi'=\xi_0$.
For each cell
$(\text{model},\text{dataset},\text{score},\xi_0,\xi')$, we run $R=10$
independent trials to account for variability from data sampling. Each
trial uses a fixed random seed to sample disjoint calibration and test
sets from the evaluation pool. By default, MMLU is stratified by subject,
with $30$ calibration and $30$ test examples per subject, while MedMCQA,
NQ, and TriviaQA use random splits with $800$ calibration and $800$ test
examples. Calibration scores are computed under $\xi_0$, the threshold
$\hat\lambda$ is set by Equation~\eqref{eq:lambda-hat}, and test scores
and prediction sets are computed under $\xi'$.

\begin{figure*}[t]
	\centering
	\begin{subfigure}[t]{0.40\textwidth}
		\centering
		\includegraphics[width=\linewidth]{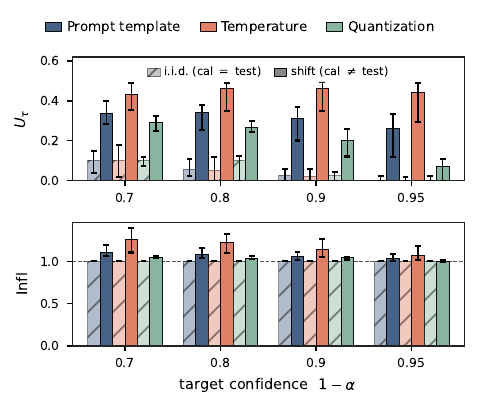}
		\caption{Target-confidence sweep at the full calibration budget.}
		\label{fig:cov-target-sweep}
	\end{subfigure}\hfill
	\begin{subfigure}[t]{0.58\textwidth}
		\centering
		\includegraphics[width=\linewidth]{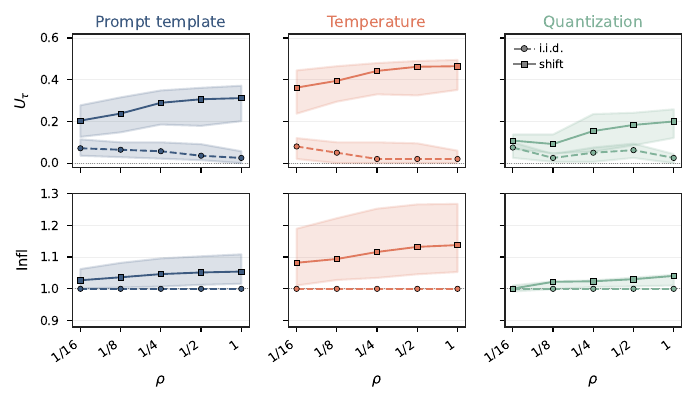}
		\caption{Calibration-budget sweep at fixed target $1-\alpha=0.9$.}
		\label{fig:ncal-sweep}
	\end{subfigure}
	\caption{\textbf{Conformalized LLMs are highly sensitive to configuration shift,
			and the failure intensifies with larger calibration budgets.}
		For each $(\text{dataset},\text{model},\text{score})$ configuration, we
		pool $10$ random calibration/test splits and compute undercoverage rate $U_\tau$ and
		set-size inflation $\mathrm{Infl}$ separately for in-distribution ($\xi'=\xi_0$) and
		shifted ($\xi'\neq\xi_0$) cells. Bars and points show medians across
		configurations; error bars and shaded bands show interquartile ranges. Colours indicate shift source, with
		hatched/dashed marks for in-distribution and solid marks for shifted
		results. \emph{Left} (\ref{fig:cov-target-sweep}): target-confidence sweep at full calibration budget. \emph{Right} (\ref{fig:ncal-sweep}): calibration-budget sweep
		$\rho=n_{\mathrm{cal}}/n_{\mathrm{cal}}^{\max}$ at fixed
		$1-\alpha=0.9$, using prefixes of the full calibration pool and a fixed
		test set.}
	\label{fig:results}
\end{figure*}

\paragraph{Evaluation metrics.}
Each evaluated cell-trial pair $a=(\xi_0,\xi',r)$ yields empirical
coverage $\widehat C_a$ and average set size $\widehat L_a$ as in
Equations~\eqref{eq:empirical-coverage}--\eqref{eq:average-set-size}.
For each $(\text{model},\text{dataset},\text{score})$ configuration,
we partition these pairs into the \emph{i.i.d.} subset
$\mathcal A_{\mathrm{iid}}=\{a:\xi'=\xi_0\}$ and the
\emph{shift} subset
$\mathcal A_{\mathrm{shift}}=\{a:\xi'\neq\xi_0\}$. 

Configuration shift can affect CP in two ways:
\emph{validity failure}, where coverage drops below the target, and
\emph{efficiency degradation}, where coverage is maintained only through
larger prediction sets. We therefore report two metrics:
\begin{itemize}[leftmargin=*,nosep]
    \item \emph{Undercoverage rate.}
    \[
    U_\tau(\mathcal A)
    =
    \frac{1}{|\mathcal A|}
    \sum_{a\in\mathcal A}\ind\{\widehat C_a<1-\alpha-\tau\}.
    \]
	This measures the fraction of runs in which empirical coverage falls more than a tolerance $\tau$ below the target.
	The tolerance  prevents sampling variability in $\widehat C_a$ from being counted as a coverage violations.
    \item \emph{Set-size inflation.}
    Let $\mathcal A^{+}=\{a\in\mathcal A:\widehat C_a\ge 1-\alpha\}$,
    and let $\bar L_{\mathrm{iid}}(\xi)$ be the mean set size among valid
    i.i.d.\ runs with deployment configuration $\xi$. We define
    \[
    \mathrm{Infl}(\mathcal A)
    =
    \frac{1}{|\mathcal A^{+}|}
    \sum_{a=(\xi_0,\xi',r)\in\mathcal A^{+}}
    \frac{\widehat L_a}{\bar L_{\mathrm{iid}}(\xi')}.
    \]
    By conditioning on $\mathcal A^{+}$, this measures efficiency only
    among valid runs, relative to each run's test-matched i.i.d.\ baseline.
\end{itemize}

\subsection{Results}
\label{sec:empirical-results}

Figure~\ref{fig:results} reports the two failure metrics under the
three configuration shifts. The left panel sweeps the target
confidence $1-\alpha\in\{0.7,0.8,0.9,0.95\}$ at the full calibration budget; the right panel
sweeps the calibration budget $\rho$ at fixed target $1-\alpha=0.9$.

\textbf{Across all target confidence levels and calibration budgets, configuration shift substantially increases the undercoverage rate $U_\tau$ relative to the in-distribution setting}. At the
$1-\alpha=0.9$, the shifted
$U_\tau$ is $0.46$ for decoding temperature, $0.31$ for prompt
template, and $0.20$ for quantization, versus an in-distribution
baseline of at most $0.03$. The in-distribution $\mathrm{Infl}$ is
$1.000$ by construction, while shifted $\mathrm{Infl}$ lies in
$[1.00,1.26]$ across all shifts, targets, and budgets, so among runs
that retain validity, \textbf{set sizes remain near their i.i.d.\
baselines}; under configuration shift the dominant failure mode is
validity loss rather than efficiency degradation. Decoding temperature
is the most damaging axis on \emph{both} metrics: it produces the
largest $U_\tau$ at every target and the only clear set-size inflation
signal, peaking at $1.26$ for temperature shift at target $0.7$.

As the calibration budget increases (Figure~\ref{fig:ncal-sweep}),
the in-distribution $U_\tau$ decays toward zero as expected,
\textbf{whereas the shifted $U_\tau$ increases with $\rho$}: $0.20\!\to\!0.31$ for prompt shift, $0.36\!\to\!0.46$ for temperature shift and $0.11\!\to\!0.20$ for quantization shift.
The trend follows from the calibration--test
mismatch. As $n_{\mathrm{cal}}$ grows, $\hat\lambda$ converges
deterministically to the $(1-\alpha)$-quantile of the
\emph{calibration} score distribution, which is biased relative to
the quantile the shifted test distribution requires for the target
coverage; larger $n_{\mathrm{cal}}$ therefore exposes this bias
rather than absorbing it. 
Due to space limitations, we defer richer fine-grained empirical results to Appendix~\ref{app:extra-results}.

\section{Statistical Guarantees under Configuration Shift}
\label{sec:bounds}

The experiments in Section~\ref{sec:empirical} show that changing only
the LLM configuration can reduce empirical coverage; we now identify
the quantity that controls this loss. CP calibrates on scores from $\xi_0$ but deploys on scores from $\xi'$, so coverage loss depends on the discrepancy between the two score distributions.

For $Z\sim P_Z$, let $P_S^{\xi}$ denote the distribution of the score
$V(Z,\xi)$, with cumulative distribution function (CDF) $F_S^{\xi}(t):=P_S^{\xi}((-\infty,t])$; write
$P_S^{\xi_0}$ and $P_S^{\xi'}$ for the calibration and test score
distributions. After calibration, the threshold $\hat\lambda$ defines
the random acceptance region $E:=(-\infty,\hat\lambda]$ in score space.
The test coverage is
\begin{equation}\label{eq:deployment-coverage}
    C := \Prob\!\bigl(S_{n+1}\le\hat\lambda\bigr)
       = \Prob\!\bigl(Y_{n+1}\in\PredSet_{\hat\lambda}(X_{n+1})\bigr)
\end{equation}
with $S_{n+1}\sim P_S^{\xi'}$. CP controls the calibration-side
quantity $\E[P_S^{\xi_0}(E)]\ge 1-\alpha$; deployment replaces
$P_S^{\xi_0}$ by $P_S^{\xi'}$, so the coverage loss is controlled by
the expected discrepancy between the two laws on the random half-line
$E$.

\paragraph{Score-space lower bounds.}
A standard discrepancy for coverage analysis under distribution shift
is total variation (TV), $\dtv(P,Q):=\sup_A |P(A)-Q(A)|$, with the
supremum over measurable sets. Adapting robust-validation arguments
\citep{cauchois2024robust,barber2023} to the score pushforward gives
a first bound.

\begin{proposition}[Score-space TV lower bound]
    \label{prop:score-tv}
    Let $S_1,\ldots,S_n\overset{\mathrm{iid}}{\sim}P_S^{\xi_0}$ be the
    calibration scores and $S_{n+1}\sim P_S^{\xi'}$ independent of
    them. Then
    \begin{equation}\label{eq:score-tv-lower}
        C \;\ge\; 1-\alpha-\dtv(P_S^{\xi_0},P_S^{\xi'}).
    \end{equation}
\end{proposition}

Proposition~\ref{prop:score-tv} is loose for CP because the
acceptance region is never an arbitrary measurable set: the threshold
$\hat\lambda$ always yields the half-line $E=(-\infty,\hat\lambda]$.
Restricting the TV supremum to half-lines gives the
Kolmogorov--Smirnov distance
\begin{equation}\label{eq:ks-def}
    \dks(P_S^{\xi_0},P_S^{\xi'})
    :=\sup_{t\in\R}\bigl|F_S^{\xi_0}(t)-F_S^{\xi'}(t)\bigr|,
\end{equation}
which upper-bounds $|P_S^{\xi_0}(E)-P_S^{\xi'}(E)|$ for every
realized calibration sample.

\begin{theorem}[Score-space KS lower bound]
    \label{thm:score-ks}
    Under the assumptions of Proposition~\ref{prop:score-tv},
    \begin{equation}\label{eq:score-ks-lower}
        C \;\ge\; 1-\alpha-\dks(P_S^{\xi_0},P_S^{\xi'}).
    \end{equation}
\end{theorem}

Since $\dks\le\dtv$, Theorem~\ref{thm:score-ks} is at least as tight
as Proposition~\ref{prop:score-tv}, and strictly tighter whenever
$\dks<\dtv$.

\paragraph{Label-free lower bounds.}
The score-space bounds are diagnostic but not fully
deployment-friendly: estimating $P_S^{\xi'}$ requires labels. We therefore relate the score-distribution shift to
a predictor state that is observable without $Y$. Let $U\in\U$ be a
measurable function of the pipeline output, for example, the
predictive simplex $\hat p_\xi(\cdot\mid X)$ for white-box scores or
the empirical sample-frequency vector for black-box scores. Write
$P_U^{\xi_0},P_U^{\xi'}$ for the distributions of $U$ under $\xi_0$
and $\xi'$. The following assumption allows the predictor-state distribution to
shift across configurations, while keeping the score distribution
conditional on the predictor state invariant.

\begin{assumption}[Kernel stability]\label{assump:kernel-stability}
    There exists a Markov kernel $K$ from $\U$ to $\R$ such that, for
    every Borel set $A\subseteq\R$,
    $P_S^{\xi_0}(A)=\int K(A\mid u)\,dP_U^{\xi_0}(u)$ and
    $P_S^{\xi'}(A)=\int K(A\mid u)\,dP_U^{\xi'}(u)$.
\end{assumption}

For $V((x,y),\xi)=\phi(\hat p_\xi(\cdot\mid x),y)$, this holds if
the conditional distribution of $Y$ given $U$ is invariant across
$\xi_0$ and $\xi'$.

\begin{theorem}[Label-free coverage lower bounds]
    \label{thm:label-free}
    Under the assumptions of Proposition~\ref{prop:score-tv} and
    Assumption~\ref{assump:kernel-stability},
    \begin{equation}\label{eq:label-free-tv}
        C \;\ge\; 1-\alpha-\dtv(P_U^{\xi_0},P_U^{\xi'}).
    \end{equation}
    If, in addition, $(\U,\rho)$ is a Polish metric space,
    $P_U^{\xi_0},P_U^{\xi'}$ have finite first moments, and
    $u\mapsto K((-\infty,t]\mid u)$ is $L$-Lipschitz uniformly in
    $t\in\R$, then
    \begin{equation}\label{eq:label-free-w1}
        C \;\ge\; 1-\alpha-L\,\Wone^{\rho}(P_U^{\xi_0},P_U^{\xi'}).
    \end{equation}
\end{theorem}

Equation~\eqref{eq:label-free-tv} follows by data processing from
Proposition~\ref{prop:score-tv}; \eqref{eq:label-free-w1} applies
Kantorovich--Rubinstein duality to each CDF gap of
Theorem~\ref{thm:score-ks}, where $\Wone^\rho$ denotes the
$1$-Wasserstein distance on $(\U,\rho)$. The Wasserstein form can be
tighter when $U$ lies in a low-dimensional continuous space, such as a
predictive simplex: nearby but disjoint predictor-state distributions
may have small transport cost even when their TV distance is large.

Full proofs are provided in Appendix~\ref{app:proofs}.

\subsection{Empirical diagnostics}
\label{sec:bounds-empirical}
The bounds in Section~\ref{sec:bounds} identify distributional discrepancies as the quantities controlling coverage
loss. We test whether finite-sample estimates of these discrepancies are
useful empirical diagnostics. For a given bound, let
$\widehat\Delta$ denote the corresponding plug-in discrepancy estimate
(e.g., the empirical KS distance for Equation~\eqref{eq:score-ks-lower})
and set $b_L=1-\alpha-\widehat\Delta$ as a plug-in lower-bound proxy. 
For each shifted cell and trial, we compute empirical coverage
$\widehat C$ (Equation~\eqref{eq:empirical-coverage}), form $b_L$, and
report the empirical conservativeness rate $\Prob(\widehat C\ge b_L)$, 
together with the average slack $\widehat C-b_L$.
Table~\ref{tab:bounds-empirical} aggregates these results across shift
cells, datasets, scores, and trials; estimator details are in
Appendix~\ref{app:bound-verification}.

We refer to the four plug-in bound proxies by the discrepancy they use:
TV$_S$, KS$_S$, TV$_U$, and $W_1$. 
\textbf{Both score-space proxies, TV$_S$ and KS$_S$, are empirically conservative on every cell in Table~\ref{tab:bounds-empirical},
and the KS$_S$  is $0.03$--$0.15$ sharper than TV$_S$ on
every shift--task pair}. 
This reflects $\dks\le\dtv$ and
realizes the half-line gain of Theorem~\ref{thm:score-ks}.

Replacing $P_S$ by the label-free predictor-state distribution $P_U$
pays the data-processing tax: TV$_U$ is typically looser than
TV$_S$ because
pushing $U$ through the score kernel can only \emph{shrink}
$\dtv(P_U^{\xi_0},P_U^{\xi'})$ down to $\dtv(P_S^{\xi_0},P_S^{\xi'})$.
Empirical conservativeness remains $\ge 0.97$; the residual failures
are consistent with mild violations of kernel stability
(Assumption~\ref{assump:kernel-stability}) and finite-sample estimation
error.
%(see Appendix~\ref{app:kernel-stability-diagnostic}).

\textbf{The predictor-state $W_1$ proxy exploits simplex geometry and is the
sharpest of all four}, with slack $0.07$--$0.23$ and margins of up to
$0.10$ over KS$_S$. This sharpness comes at the cost of lower empirical
conservativeness ($0.85$--$0.99$), 
reflecting the additional smoothness condition in
Theorem~\ref{thm:label-free} and our heuristic sliced-$W_1$, $L=1$.
\paragraph{Takeaway.}
When test labels are available, KS$_S$ is the most reliable plug-in
diagnostic in our experiments: it is empirically conservative on every
aggregate cell and uniformly sharper than TV$_S$. In the realistic
deployment regime where test labels are unavailable, $W_1$ provides
the sharpest heuristic signal, trading a small conservativeness slip
($\le 15\%$ in the worst cells) for a substantial sharpness gain over
the label-free TV$_U$.

\begin{table}[t]
    \centering
    \footnotesize
    \setlength{\tabcolsep}{4pt}
    \begin{tabular}{l cccc}
        \toprule
        Shift & TV$_S$ & KS$_S$ & TV$_U$ & $W_1$ \\
        \midrule
        \multicolumn{5}{l}{\emph{MCQA (MMLU, MedMCQA)}}\\
        prompt        & .26\,(1.00) & .14\,(1.00) & .33\,(1.00) & .13\,(.92) \\
        temperature   & .35\,(1.00) & .20\,(1.00) & .34\,(.97)  & .15\,(.88) \\
        quantization  & .18\,(1.00) & .07\,(1.00) & .22\,(1.00) & .07\,(.99) \\
        \midrule
        \multicolumn{5}{l}{\emph{OEQA (NQ, TriviaQA)}}\\
        prompt        & .29\,(1.00) & .22\,(1.00) & .31\,(.99)  & .16\,(.90) \\
        temperature   & .36\,(1.00) & .33\,(1.00) & .51\,(1.00) & .23\,(.85) \\
        quantization  & .28\,(1.00) & .18\,(1.00) & .30\,(1.00) & .14\,(.99) \\
        \bottomrule
    \end{tabular}
    \caption{\textbf{Empirical evaluation of plug-in lower-bound proxies at $1-\alpha=0.9$},
    	aggregated over shift cells
    	($\xi'\neq\xi_0$) of Section~\ref{sec:empirical}, pooled
    	across model, dataset, score, and trial. Each cell reports
    	average slack $\widehat C-b_L$ with empirical conservativeness rate
    	$\Prob(\widehat C\ge b_L)$ in parentheses, where
    	$b_L=1-\alpha-\widehat\Delta$. }\label{tab:bounds-empirical}
\end{table}

\section{Mitigating Configuration Shift}
\label{sec:mitigations}

We evaluate six mitigations organized by where each comes from: two
are adapted from existing work (\emph{Reweight}, \emph{Mos-U}), two
are derived from the bounds of Section~\ref{sec:bounds}
(\emph{$\alpha$-Inf}, \emph{Recal}), and two are motivated by the
per-configuration fragility patterns of Section~\ref{sec:empirical}
(\emph{Mos-F}, \emph{Anc}). All six share the calibration data
$\Dcal$, target $1-\alpha$, and test configuration $\xi'$ and differ only in how the threshold or the nonconformity scores are constructed. Full method specifications and implementation details are in
Appendix~\ref{app:mitigation-detail}.

\subsection{Methods}
\label{sec:mit-methods}

\paragraph{Adapted from prior work (Reweight, Mos-U).}
\emph{Reweight} ports weighted split CP
\citep[Section 2.3]{tibshirani2019covariate} to predictor-state
space: a logistic discriminator $\hat\pi$ between calibration and
deployment predictor states yields importance weights
$w_i=\hat\pi/(1-\hat\pi)$, and $\hat\lambda$ is the $w$-weighted
$(1-\alpha)$-quantile of the calibration scores.
\emph{Mos-U} is the calibration analogue of prompt ensembling \cite{kumar2023mcq}: calibration scores are computed under multiple configurations drawn uniformly from an augmentation pool of alternative configurations.

\paragraph{Bound-inspired ($\alpha$-Inf, Recal).}
Both methods turn the predictor-state bound of
Theorem~\ref{thm:label-free} into an actionable correction by
estimating
$\widehat\varepsilon = L\cdot\widehat\Wone(\hat P_U^{\xi_0},\hat P_U^{\xi'})$
from empirical predictor distributions. 
\emph{$\alpha$-Inf} deflates the miscoverage budget to $\alpha'=\max(\alpha-\widehat\varepsilon,0)$, raising the calibration quantile to $1-\alpha'$ to pre-pay the shift-induced coverage gap. 
\emph{Recal} is the gated variant: when $\widehat\varepsilon$ exceeds a noise floor $\beta$ estimated under no shift, it spends a small labeled deployment budget $k$ to recompute $\hat\lambda$ directly; otherwise it falls back.

\paragraph{Fragility-aware (Mos-F, Anc).}
Both extend Mos-U on the principle that calibrating against harder
cases yields more conservative thresholds
\citep{aolaritei2025lp,luo2024gametheory}. Each candidate
configuration is scored off-line by how often vanilla CP undercovers
it on the same configuration grid as Section~\ref{sec:empirical} but auxiliary
audit splits disjoint from the final mitigation trials
(Figure~\ref{fig:fragility-profile} illustrates the resulting profile). \emph{Mos-F} samples
each calibration configuration with probability proportional to its
undercoverage rate; \emph{Anc} restricts the pool to the reference
configuration plus the top-$K$ configurations with the highest
undercoverage rates. 

\begin{figure}[t]
    \centering
    \includegraphics[width=0.85\linewidth]{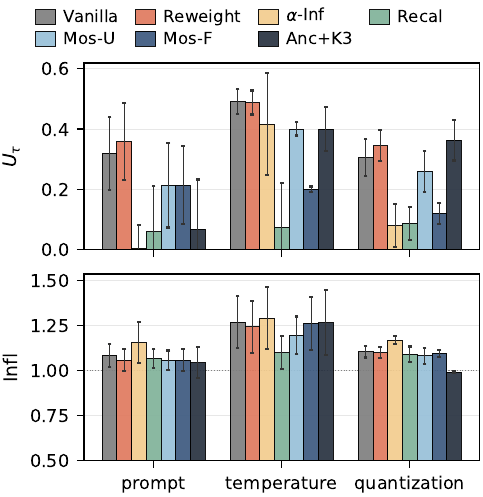}
    \caption{\textbf{Mitigation comparison across the three shift
    axes.} Median undercoverage rate $U_\tau$ (top) and
    set-size inflation $\mathrm{Infl}$ (bottom) at $1-\alpha=0.9$, across
    $(\text{model},\text{dataset},\text{score})$ configurations. Error
    bars are cross-configuration standard deviations. The
    \emph{Anc+K3} label denotes \emph{Anc} instantiated with $K=3$
    (a $K$-ablation is deferred to Appendix~\ref{app:mitigation-detail}).}
    \label{fig:mitigations}
\end{figure}

\subsection{Empirical Results}
\label{sec:mit-results}

We report the same metrics as
Section~\ref{sec:empirical}. Figure~\ref{fig:mitigations} stacks both metrics for
the seven methods on the three shifts.

$\alpha$-Inf drives undercoverage close to zero on prompt
($0.319{\to}0.004$) and well below vanilla on quantization
($0.305{\to}0.080$), but the predictor-state Lipschitz bound is
loose enough that $\widehat\varepsilon$ pushes $\alpha'$ near zero
and prediction sets grow by $15$--$17\%$ on those axes. The
resulting predictor is valid but inefficient. Recal replaces
deflation with a gated, label-budgeted threshold recalibration that fires
only when $\widehat\varepsilon$ exceeds the no-shift noise floor,
and is the best method overall: vanilla undercoverage of
$0.319/0.490/0.305$ drops to $0.060/0.075/0.087$ at inflation cost
within $0.04$ of vanilla (and \emph{below} vanilla on temperature,
$1.270{\to}1.100$). A small deployment-side label set ($k{=}50$)
spent only where the gate fires is therefore enough to absorb most
of the shift damage; a budget ablation over
$k\in\{10,20,30,40,50,60\}$ (Appendix~\ref{app:mitigation-detail},
Figure~\ref{fig:recal-budget}) shows undercoverage saturating within a
few dozen labels.

Uniform sampling from the configuration pool (\emph{Mos-U}) cuts vanilla undercoverage by $0.05$--$0.11$ on every shift with no inflation cost. This is the central cold-start
result: a label-free calibration ensemble already absorbs some of
the shift damage. \emph{Mos-F} additionally over-samples
configurations on which vanilla CP undercovers; the gain over
Mos-U is large on temperature ($0.400{\to}0.200$) and quantization
($0.260{\to}0.120$). \emph{Anc} reaches
Recal-level performance on prompt ($0.069$, inflation $1.044$) but in general does not improve.

Weighted CP fails because it assumes covariate shift with a fixed score function \citep{tibshirani2019covariate}. Under configuration shift, however, the deployed score $V(\cdot,\xi')$ differs from the calibration score $V(\cdot,\xi_0)$ because the predictor changes with the configuration. 

\paragraph{Practical recommendation.}
When even a small deployment-side label set is available or
accumulates over time, Recal is the clear choice: it nearly closes
the undercoverage gap on every shift at negligible inflation cost.
Under a cold-start with no test labels, a uniform calibration-side
configuration ensemble (Mos-U) already absorbs much of the shift
damage at zero inflation cost; if an off-line undercoverage profile
of the configuration grid is available, Mos-F adds another
$0.05$--$0.20$ in undercoverage reduction on shift axes where
fragility concentrates.

\section{Related Work}
\label{sec:related}

%Appendix~\ref{app:related-extended} gives an extended discussion, including connections to uncertainty quantification for LLMs and to shift detection.

\paragraph{CP under distribution shift.}
Most extensions of CP beyond exchangeability model shifts in
the data law. This includes covariate-shift reweighting and adaptive
online updates \citep{tibshirani2019covariate,gibbs2021adaptive},
general non-exchangeable sequences \citep{barber2023}, and robust
prediction sets \citep{cauchois2024robust,xu2025wrcp,yang2026multidist,aolaritei2025lp}.
Recent work has also studied conformalized LLMs under covariate
shift \citep{hu2026cofact}. These methods protect coverage when the
examples or their weights change, typically with a fixed score map. Our setting is complementary: configuration shift keeps $P_Z$ fixed
but changes the score law $P_S^\xi$.
Our aim is \textbf{not} to propose a universal robust-CP method, but to expose configuration shift as a distinct failure mode, connect its coverage loss to score-distribution discrepancies, and provide initial diagnostics and mitigations.

\paragraph{CP for LLMs.}
CP has been adapted to various LLM tasks
\citep{kumar2023mcq,su2024api,quach2024clm,mohri2024factuality}. These papers define valid sets for a chosen LLM scoring pipeline. LLM pipeline choices are typically fixed as part of the experimental protocol, or used for auxiliary sensitivity analysis rather than studied
as a source of score-level shift \citep{kumar2023mcq,su2024api,wang2024conu,wang2025sconu}.
Our findings give a different picture: routine configuration changes
can shift the score distribution enough to break validity.

\paragraph{Pipeline sensitivity and ensembling.}
Separate from CP, prompt-format and inference-pipeline sensitivity are
well documented in LLMs
\citep{sclar2024quantifying,shi-etal-2024-thorough,helcig2026statistically}.
Prompt ensembling methods reduce such sensitivity by aggregating
predictions or probabilities across templates
\citep{jiang2023cape,tonolini2024bayespe}. 
In CP, stability of predictions is not sufficient. Our work motivates calibration ensembles designed and evaluated in terms of CP validity.

\section{Conclusion}
This paper identifies configuration shift as an overlooked source of CP failure in LLM pipelines. We show that even when the underlying examples are unchanged, routine changes to the inference configuration can alter the nonconformity score distribution enough to break the exchangeability on which split CP relies. Across models, datasets, and scores, this failure appears primarily as undercoverage rather than large set inflation, making it easy to miss if configuration changes are treated as harmless implementation details.

The broader lesson is that CP guarantees for LLMs should be understood as guarantees for a complete score-generating pipeline, not for a dataset or model in isolation. As CP moves into retrieval-augmented, agentic, and decision-making LLM systems
\citep{feng2025conformalrag,ren2023knowno,vishwakarma2025prune}, configuration
shift should be treated as a first-class deployment risk alongside data shift.
Future conformalized-LLM work should state the configurations under which
guarantees are claimed, stress-test realistic pipeline changes, and calibrate
over the settings where coverage is expected to hold.

\section{Limitations}
First, our bounds are not finite-sample certificates. The theoretical results
are population statements, while the experiments use finite-sample plug-in
estimates of the relevant discrepancies. We therefore interpret these estimates
as diagnostics of shift severity rather than certified lower bounds on
coverage. Obtaining certified bounds would require replacing the population
discrepancy by a high-probability upper bound; for one-dimensional score-space
KS terms, this could be done using Dvoretzky--Kiefer--Wolfowitz concentration
\citep{dvoretzky1956asymptotic,massart1990tight}. Deriving sharp and practical
finite-sample certificates for these discrepancies is an important but separate
direction.

Second, our mitigations are initial and deliberately simple. They are designed
to test whether the discrepancy quantities identified by our analysis can guide
practical responses to configuration shift, not to provide a universal robust-CP
method for arbitrary deployment shifts. Our goal is to expose configuration
shift as a distinct failure mode, connect its coverage loss to
score-distribution discrepancies, and provide initial diagnostics and
mitigations. A more complete robust calibration framework for
configuration-indexed score maps remains an important direction for future
work.

Third, our experiments focus on controlled shift axes: prompt templates,
sampling temperature and weight quantization. This design makes it
possible to isolate how each source of shift affects CP, but
it does not exhaust all deployment changes, such as simultaneous multi-axis
shifts or changes in retrieval, system prompts, or data pipelines. Extending
the analysis to such compound shifts is a natural next step.

\section{Potential Risks}
Our work does not introduce a new model, dataset, or deployed system. The main
risk is overinterpretation: practitioners may treat our empirical diagnostics
and mitigation results as certified robustness of CP under arbitrary pipeline or
data shifts. This could create unwarranted confidence in conformalized LLM
systems, especially in high-stakes deployments. We mitigate this risk by
emphasizing that CP guarantees are configuration-specific, our bounds are
diagnostic rather than finite-sample certificates, and deployment claims should
be tied to explicit configurations and stress tests.

\section{Acknowledgements}
The authors thank the International Max Planck
Research School for Intelligent Systems (IMPRS-IS) for supporting Yuqicheng Zhu. 
This work was conducted during a research visit to the University
of Oxford, supported by the ELIAS Mobility Programme (GA 101120237) and EU Projects SMARTY (GA 101140087). The work was also partially supported by EU projects euroFMX (GA 101299128) and CEREBRA-AI (GA 101299046).
We thank Prof. Ian Horrocks for hosting the research visit and the Department of Computer Science for its institutional support.

\bibliography{custom}

\appendix
\newpage
\section{Nonconformity Scores for Conformalized LLMs}
\subsection{White-box Nonconformity Scores}
\label{app:whitebox-scores}

The two white-box scores of Section~\ref{sec:cp-llm}, LAC and the logit
margin, are both deterministic functions of the model's predictive
distribution $\hat p_\xi(\cdot\mid x)$ over the label set $\Y$. This appendix
explains how that distribution is formed under an inference configuration
$\xi$, the predictor state it induces, and why the white-box scores are
restricted to the multiple-choice setting.

\paragraph{Logit-based predictive distribution.}
For MCQA the label set is the fixed option set
$\Y=\{A,B,C,D\}$. Under configuration $\xi$ we query the model once and read
the token-level log-probabilities $\ell_\xi(k\mid x)$ it assigns to the four
option symbols, then renormalise them by a softmax over the option set,
\[
  \hat p_\xi(k\mid x)
  =\frac{\exp\ell_\xi(k\mid x)}{\sum_{k'\in\Y}\exp\ell_\xi(k'\mid x)},
  \qquad k\in\Y.
\]
For the temperature-shift experiments, $\ell_\xi(k\mid x)$ is computed by
applying the swept temperature to the final-token logits before this
renormalisation, i.e.\ by using log-probabilities from
$\mathrm{softmax}(\mathrm{logits}/T)$; for the prompt and quantization sweeps
we use $T=1$.
Both white-box scores are functions of
this vector alone: LAC is $1-\hat p_\xi(y\mid x)$ and the logit margin is
$\max_{k\neq y}\{\log\hat p_\xi(k\mid x)-\log\hat p_\xi(y\mid x)\}$, so the
predictive distribution is a sufficient summary of the model output for both.

The construction above requires a fixed, finite label set $\Y$ over which the softmax is defined. OEQA has no such set. The white-box scores are therefore evaluated only on the MCQA datasets (MMLU,
MedMCQA); on the open-ended datasets (NQ, TriviaQA) we use the black-box scores of Appendix~\ref{app:blackbox-scores}, which depend only on sampled text.

\subsection{Black-box Nonconformity Scores}
\label{app:blackbox-scores}

The white-box scores of Section~\ref{sec:cp-llm} (LAC and the logit
margin) are computed from the model's predictive distribution
$\hat p_\xi(\cdot\mid x)$ over candidate labels. This distribution is
unavailable in two situations that are common for LLMs: (i)~the model
is served behind an API that returns only generated text, with no
logits or token probabilities; and (ii)~the task is open-ended, so
there is no fixed finite label set over which a softmax is defined.
The two \emph{black-box} scores used in this paper, \emph{self-consistency}
and \emph{LoFreeCP}, are designed for these settings: they replace the
predictive distribution by an empirical one estimated from sampled
generations, and therefore depend only on the model's text output.

\paragraph{Sampling-based predictive distribution.}
Fix an input $x$ and an inference configuration $\xi$. We draw $M$
independent generations $Y^{(1)},\ldots,Y^{(M)}$ from the model under
$\xi$ (i.e.\ with stochastic decoding) and summarise them by the
empirical frequency of each candidate $y$,
\[
F_y^\xi \;=\; \frac{1}{M}\sum_{j=1}^M \ind\!\bigl(Y^{(j)}=y\bigr).
\]
Generations that cannot be parsed into a valid candidate are discarded,
and $M$ counts only the valid generations. The frequency vector
$F^\xi=(F_y^\xi)_y$ is a Monte-Carlo estimate of the model's predictive
distribution. This same frequency vector is the
predictor state $U$ used by the label-free bounds of
Theorem~\ref{thm:label-free} for black-box scores.
In the reported runs we sample $20$ completions per example by default. 
Decoding uses no top-$k$ truncation ($\texttt{top\_k}=0$), no nucleus
truncation beyond the full distribution ($\texttt{top\_p}=1.0$), and maximum
lengths of $8$ new tokens for MCQA and $24$--$48$ for OEQA. Random seeds are
used only to make cached generations reproducible; samples are not shared or
coupled across configurations.

\paragraph{Candidate set.}
What counts as a candidate $y$, and when two generations are treated as
the same candidate, depends on the task.
\begin{itemize}[leftmargin=*,nosep]
	\item \emph{MCQA.} The candidate set is
	the fixed option set $\Y=\{A,B,C,D\}$. Each generation is parsed to
	the first option letter it contains, and $\ind(Y^{(j)}=y)$ is an
	exact match on that parsed letter, so $F_y^\xi$ is the fraction of
	parsed generations equal to option $y$.
	\item \emph{OEQA.} There is no predefined
	label set, so the candidate set is the collection of \emph{distinct
		answer strings that actually appear} among the $M$ generations. Two
	generations are treated as the same candidate when their normalised
	forms coincide, using the SQuAD-style answer normalisation of
	\citet{su2024api} and \citet{quach2024clm} (lowercasing, removal of
	articles and punctuation, whitespace collapse, and truncation at the
	first sentence delimiter); $\ind(Y^{(j)}=y)$ is then exact match on
	these normalised strings.
\end{itemize}
In both cases $F^\xi$ is a distribution over the realised candidates
only: a candidate that the model never produced has $F_y^\xi=0$ and, in
the open-ended case, is simply absent from the candidate set.

\paragraph{Self-consistency.}
\emph{Self-consistency} \citep{wang2023selfconsistency,su2024api} scores
a candidate by one minus its sampling frequency,
\[
V((x,y),\xi) = 1 - F_y^\xi,
\]
so that candidates appearing more often in the samples are treated as
more conforming. It is the direct sampling analogue of LAC
($1-\hat p_\xi(y\mid x)$): replacing the softmax probability by its
Monte-Carlo estimate $F_y^\xi$ recovers exactly this score.

\paragraph{LoFreeCP.}
\emph{LoFreeCP} \citep{su2024api} augments the frequency signal with two
correction terms,
\[
V((x,y),\xi) = -F_y^\xi
+ \lambda_1 \cdot H_{\mathrm{norm}}^\xi
- \lambda_2 \cdot \mathrm{sim}(y, \hat y^\xi),
\]
with non-negative weights $\lambda_1,\lambda_2$ (we use the defaults
$\lambda_1=\lambda_2=1$). The three terms play complementary roles.
\begin{itemize}[leftmargin=*,nosep]
	\item \emph{Frequency} ($-F_y^\xi$). The same consensus signal as
	self-consistency: more frequently sampled candidates are more
	conforming. Up to the sign convention, self-consistency is exactly
	LoFreeCP with $\lambda_1=\lambda_2=0$, so the two scores are nested
	and self-consistency isolates the contribution of the frequency term.
	\item \emph{Question-level uncertainty} ($+\lambda_1 H_{\mathrm{norm}}^\xi$).
	$H_{\mathrm{norm}}^\xi = -\sum_y F_y^\xi\log F_y^\xi \,/\, \log M$ is
	the Shannon entropy of the empirical frequency vector, normalised by
	$\log M$ so that $H_{\mathrm{norm}}^\xi\in[0,1]$ ($\log M$ is the
	largest possible entropy, attained when all $M$ generations are
	distinct). 
	\item \emph{Semantic agreement} ($-\lambda_2\,\mathrm{sim}(y,\hat y^\xi)$).
	$\hat y^\xi = \arg\max_y F_y^\xi$ is the most frequently sampled
	(top-1) candidate, and $\mathrm{sim}(y,\hat y^\xi)$ is the cosine
	similarity between sentence embeddings of $y$ and $\hat y^\xi$ (we
	embed answer strings with \texttt{all-MiniLM-L6-v2}); the top-1
	candidate itself is assigned $\mathrm{sim}(\hat y^\xi,\hat y^\xi)=0$ by
	convention. 
\end{itemize}

\paragraph{Calibration and set construction.}
For multiple-choice QA the score is a fixed $|\Y|$-way vector per
question, and calibration and set construction proceed exactly as in
Section~\ref{sec:cp-background}: the calibration score is $V$ evaluated
at the true option, and the prediction set is
$\{y\in\Y: V((x,y),\xi)\le\hat\lambda\}$. For open-ended QA the
candidate set is the per-question sample pool, and we follow the
LoFreeCP protocol \citep{su2024api}: the calibration score of a question
is the smallest $V$ over pool entries that match the gold answer (and
$+\infty$ if the model never sampled a correct answer); the prediction
set keeps every pool entry with $V\le\hat\lambda$; and a test question
is counted as covered when at least one retained entry matches the gold
answer.

\section{Details of Datasets and Models}
\label{app:artifact-details}

\subsection{Datasets and Data Statistics}
\label{app:data-details}

We use four public English question-answering benchmarks, citing the
dataset creators in Section~\ref{sec:exp-setting}. MMLU
\citep{hendrycks2021mmlu} and MedMCQA \citep{pal2022medmcqa} are used
as multiple-choice QA datasets with a fixed four-option label space. Natural
Questions \citep{kwiatkowski2019nq} and TriviaQA
\citep{joshi2017triviaqa} are used as open-ended QA datasets with
string-valued answers and aliases. We use these artifacts only for research
evaluation of uncertainty sets, consistent with their standard benchmark use,
and do not collect new human data or create new labels.

Table~\ref{tab:dataset-artifacts} summarizes the dataset artifacts used in
our experiments. The public artifact repositories provide the dataset
documentation and licenses; at the time of use, MMLU and MedMCQA were listed
under MIT-style licenses, Natural Questions under CC-BY-SA-3.0, and the
original TriviaQA repository under Apache-2.0. Because these are public QA
benchmarks, they may contain named entities, web-derived text, medical exam
vignettes, or other source-inherited sensitive or offensive content. We do not
use the datasets for profiling individuals or making real-world decisions.

\begin{table*}[t]
    \centering
    \small
    \begin{tabular}{p{0.13\textwidth}p{0.22\textwidth}p{0.18\textwidth}p{0.34\textwidth}}
        \toprule
        Dataset & Artifact description & Documentation/license & Statistics used in this paper \\
        \midrule
        MMLU & Broad four-option MCQA benchmark over $57$ subjects. & Dataset paper/card; MIT-style license. & Stratified by subject with $30$ calibration and $30$ test examples per subject in each trial. \\
        MedMCQA & Medical-entrance-exam MCQA benchmark covering clinical and biomedical subjects. & Dataset paper/card; MIT-style license. & Random splits with $800$ calibration and $800$ test examples in each trial. \\
        NQ & Open-domain QA benchmark built from real search questions and Wikipedia-based answers. & Dataset paper/card; CC-BY-SA-3.0. & Random splits with $800$ calibration and $800$ test examples in each trial; OEQA evaluation is conditional on gold-in-candidate retention. \\
        TriviaQA & Factoid open-ended QA benchmark with answer aliases and web/Wikipedia evidence in the original artifact. & Dataset paper/card; Apache-2.0 in the original repository. & Random splits with $800$ calibration and $800$ test examples in each trial; OEQA evaluation is conditional on gold-in-candidate retention. \\
        \bottomrule
    \end{tabular}
    \caption{\textbf{Dataset artifacts used in the empirical study.}
    All experiments use public benchmark data for research evaluation only.
    Additional evaluation denominators, base accuracies, and gold-in-candidate
    retention rates are reported in Table~\ref{tab:evaluation-diagnostics}.}
    \label{tab:dataset-artifacts}
\end{table*}

\subsection{Models and Model Artifacts}
\label{app:model-details}

We evaluate nine open-weight instruction models from three families:
Llama-3.2-Instruct 1B/3B and Meta-Llama-3-8B-Instruct
\citep{meta2024llama32,grattafiori2024llama3}; Gemma-3-it 1B/4B/12B
\citep{gemmateam2025gemma3}; and Qwen3.5 0.8B/4B/9B
\citep{qwen3.5}. The model cards and technical reports document the intended
uses, training summaries, and safety or use-policy restrictions for these
artifacts. We use the models only for automated inference on benchmark
questions; we do not fine-tune, redistribute, or deploy them.

For the quantization-shift experiments, we additionally use GGUF quantized
checkpoints for Qwen3.5-9B, Meta-Llama-3-8B-Instruct, and Gemma-3-12B-it
\citep{bartowski_qwen35_9b_gguf,quantfactory_llama3_8b_instruct_gguf,
unsloth_gemma3_12b_it_gguf}. These checkpoints instantiate the deployment
configuration axis studied in Section~\ref{sec:exp-setting}; their repository
cards document the available GGUF precisions. We use four precisions
(\texttt{Q8}, \texttt{Q6}, \texttt{Q4}, and \texttt{Q2}) as evaluation
conditions and follow the upstream model terms: Meta Llama community licenses
for Llama models, Gemma Terms of Use for Gemma models, and Apache-2.0 for
Qwen models.

\section{Prompt Templates}
\label{app:prompt-templates}

This appendix records the prompt-template suites used in the prompt-template
shift sweeps of Section~\ref{sec:exp-setting}. Two template suites are used:
a \textbf{14-template MCQA suite}, shared verbatim by MMLU and MedMCQA, and a
\textbf{7-template OEQA suite}, shared verbatim by Natural Questions and
TriviaQA. Within each suite, every template is one of three categories:
\begin{enumerate}[leftmargin=*,nosep]
  \item \emph{Surface format} --- vary how the answer options (MCQA) or
        question/answer prefixes (OEQA) are typeset; instructions and
        demonstration content are held fixed.
  \item \emph{Instruction syntax and wording} --- vary the verbosity, persona,
        and phrasing of the task instructions; demonstrations and surface
        format are held fixed.
  \item \emph{Content ablation} --- remove exactly one scaffolding component
        (subject preamble, chain-of-thought directive, or few-shot demos)
        from the anchor template.
\end{enumerate}
Baseline prompts are displayed below; non-anchor templates are specified as
deltas from the baseline. The canonical templates that the model receives are
stored in the per-dataset YAML configuration files under the project's
\texttt{experiments/} directory.

\subsection{Template schema}
\label{app:templates-schema}

Each template is a Python format string with task-specific placeholders;
placeholders that a particular template does not use are silently dropped at
format time. Concrete prompts shown below contain the raw placeholders in
braces (e.g.\ \texttt{\{question\}}) so the template structure is visible.

\paragraph{MCQA placeholders.}
\texttt{\{subject\}} is a pretty-printed subject name (e.g.\ ``abstract
algebra''); \texttt{\{question\}} is the test question; \texttt{\{choices\}}
is the four options rendered in the template's \texttt{choice\_style}; and
\texttt{\{demo\_question\}}, \texttt{\{demo\_choices\}},
\texttt{\{demo\_answer\}} carry the one-shot demonstration (a held-out
dev-split example with its correct option label). A \texttt{choice\_style} is
a (\emph{delimiter}, \emph{label set}) pair drawn from delimiters
\{\texttt{paren} ``(A) text'', \texttt{dot} ``A. text'',
\texttt{dash} ``A) text''\} and label sets
\{\texttt{upper} (A--D), \texttt{lower} (a--d), \texttt{numeric} (1--4)\}.

\paragraph{OEQA placeholders.}
\texttt{\{question\}} is the test question; \texttt{\{demos\}} is the
few-shot demonstration block; and \texttt{\{q\_prefix\}}, \texttt{\{a\_prefix\}}
are the per-template question and answer prefix tokens (default
\texttt{``Q: ''} / \texttt{``A: ''}). \texttt{\{demos\}} renders each
demonstration as the two-line block ``\texttt{\{q\_prefix\}\{q\}}\textbackslash n
\texttt{\{a\_prefix\}\{a\}}'' and joins successive demonstrations with a
blank line; the same prefix tokens are substituted into demonstrations and
the test question so the prompt is internally consistent. The five
demonstration QA pairs (Table~\ref{tab:oeqa-demos}) are taken verbatim from
the TriviaQA prompt of~\citet{su2024api} (Appendix~C.1) and are shared by
both OEQA datasets.

\subsection{MCQA template suite (MMLU, MedMCQA)}
\label{app:mcqa-templates}

The MCQA suite contains 14 templates: 9 surface-format variants (a full
$3\times 3$ delimiter $\times$ label-set grid), 3 instruction-syntax
variants (counting the anchor), and 3 content-ablation variants. All
templates are shared verbatim between MMLU and MedMCQA; only the
dataset-specific values of \texttt{\{subject\}}, \texttt{\{question\}}, and
the demonstration placeholders differ.

\paragraph{Anchor: \textsc{Baseline}.}
The anchor is the paper-faithful one-shot prompt of~\citet{kumar2023mcq}:
a subject preamble, a fully rendered one-shot example with its correct
answer, a chain-of-thought (CoT) directive, and the test question terminated
with the cue ``The correct answer is option:''. Options are rendered in
\texttt{paren\_upper} style (``(A)~text''). The template body is shown in
Figure~\ref{fig:mcqa-baseline}.

\begin{figure}[!t]
\footnotesize
\begin{verbatim}
This is a question from {subject}.

{demo_question}
{demo_choices}
The correct answer is option: {demo_answer}.

Reason step-by-step and answer the following
question.
{question}
{choices}
The correct answer is option:
\end{verbatim}
\caption{MCQA \textsc{Baseline} template. \texttt{\{choices\}} and
\texttt{\{demo\_choices\}} are rendered in \texttt{paren\_upper}
(``(A)~text \ldots (D)~text''); the demonstration is drawn from
dev-split index $0$.}
\label{fig:mcqa-baseline}
\end{figure}

\paragraph{Category 1 --- Surface format (9 variants).}
The 9 surface-format variants reuse the \textsc{Baseline} prompt body
verbatim and only change the \texttt{choice\_style}: the full $3\times 3$
grid of delimiters \{\texttt{paren}, \texttt{dot}, \texttt{dash}\} crossed
with label sets \{\texttt{upper}, \texttt{lower}, \texttt{numeric}\}. The
combination \texttt{paren\_upper} coincides with the anchor, so there are
$8$ \emph{non-anchor} surface-format templates; we retain
\texttt{paren\_upper} in the count of $9$ so the grid is a complete
factorial. Each style is summarised in Table~\ref{tab:mcqa-styles}.

\begin{table}[!t]
\footnotesize
\centering
\begin{tabular}{lll}
\toprule
\textbf{Style} & \textbf{Format} & \textbf{Example} \\
\midrule
\texttt{paren\_upper} (anchor) & \texttt{(\{l\}) \{t\}} & \texttt{(A) text} \\
\texttt{paren\_lower}          & \texttt{(\{l\}) \{t\}} & \texttt{(a) text} \\
\texttt{paren\_numeric}        & \texttt{(\{l\}) \{t\}} & \texttt{(1) text} \\
\texttt{dot\_upper}            & \texttt{\{l\}. \{t\}}  & \texttt{A. text}  \\
\texttt{dot\_lower}            & \texttt{\{l\}. \{t\}}  & \texttt{a. text}  \\
\texttt{dot\_numeric}          & \texttt{\{l\}. \{t\}}  & \texttt{1. text}  \\
\texttt{dash\_upper}           & \texttt{\{l\}) \{t\}}  & \texttt{A) text}  \\
\texttt{dash\_lower}           & \texttt{\{l\}) \{t\}}  & \texttt{a) text}  \\
\texttt{dash\_numeric}         & \texttt{\{l\}) \{t\}}  & \texttt{1) text}  \\
\bottomrule
\end{tabular}
\caption{Surface-format \texttt{choice\_style}s used in the MCQA suite.
Each row defines a (delimiter, label-set) pair; \texttt{\{l\}} is one of
four label tokens (e.g.\ ``A'', ``a'', or ``1'') and \texttt{\{t\}} is the
option text. The surrounding prompt body is identical for all $9$ variants
(Figure~\ref{fig:mcqa-baseline}).}
\label{tab:mcqa-styles}
\end{table}

\paragraph{Category 2 --- Instruction syntax / wording (3 variants).}
We vary instruction wording while keeping the demonstration content and
answer-format expectations fixed. The \textsc{Baseline} body sits at the
middle of this axis. The \texttt{formal} variant uses \texttt{dot\_upper}
options and an expert-role prompt with an
\texttt{Example:/Question:/Options:/Answer:} demonstration followed by the
test schema \texttt{Question:/Options:/Answers:}. The \texttt{minimal}
variant uses \texttt{dash\_upper} options and consists only of
\texttt{\{question\}}, \texttt{\{choices\}}, and the answer cue
\texttt{The correct answer is option:}.

\paragraph{Category 3 --- Content ablation (3 variants).}
The content-ablation templates each remove exactly one scaffolding component
from \textsc{Baseline} while leaving everything else (option style,
demonstration content, answer cue) unchanged: \texttt{no\_subject} drops the
``This is a question from \{subject\}.'' preamble; \texttt{no\_cot} replaces
``Reason step-by-step and answer the following question.'' with the shorter
``Answer the following question.''; \texttt{zero\_shot} drops the one-shot
demonstration block. All three retain \texttt{paren\_upper} options and use
dev-split index $0$ where applicable.

\subsection{OEQA template suite (NQ, TriviaQA)}
\label{app:oeqa-templates}

The OEQA suite contains $7$ templates: $3$ surface-format variants (counting
the anchor), $3$ instruction-syntax variants (counting the anchor), and
$2$ content-ablation variants. All templates are shared verbatim between
Natural Questions and TriviaQA. Unlike the MCQA suite, the OEQA anchor is
already an elaborate prompt; it follows the prompt shape of the LoFreeCP
paper~\citep{su2024api}, which is the published reference for the LoFreeCP
score on open-ended QA. Cross-task comparisons between the MCQA and OEQA
anchors should therefore be read with this asymmetry in mind.

\paragraph{Anchor: \textsc{Baseline}.}
The anchor is a 5-shot prompt with a task-description preamble and a
chain-of-thought directive. Question and answer lines carry the prefixes
\texttt{``Q: ''} and \texttt{``A: ''}. The same prefix tokens are used to
render each demonstration (so each demo expands to a \texttt{Q:\ \ldots}
line followed by an \texttt{A:\ \ldots} line, with successive demos
separated by a blank line). The five demonstration QA pairs are listed in
Table~\ref{tab:oeqa-demos}. The template body is shown in
Figure~\ref{fig:oeqa-baseline}.

\begin{figure}[!t]
\footnotesize
\begin{verbatim}
Please engage in the question-answering task.
You should generate a short factual answer to each
question. Examples are provided.

{demos}

Reason step-by-step and answer the following
question with a short factual answer.

{q_prefix}{question}
{a_prefix}
\end{verbatim}
\caption{OEQA \textsc{Baseline} template.
\texttt{\{demos\}} expands to the five
demonstration QA pairs of Table~\ref{tab:oeqa-demos}, each rendered
as ``\texttt{Q: $\langle q\rangle$}\textbackslash n
\texttt{A: $\langle a\rangle$}'' and separated by a blank line.
\texttt{\{q\_prefix\}} and \texttt{\{a\_prefix\}} are
\texttt{``Q: ''} and \texttt{``A: ''} for the anchor.}
\label{fig:oeqa-baseline}
\end{figure}

\begin{table}[!t]
\footnotesize
\centering
\begin{tabular}{p{0.62\linewidth} l}
\toprule
\textbf{Question} & \textbf{Answer} \\
\midrule
Which American-born Sinclair won the Nobel Prize for Literature in 1930? & Sinclair Lewis \\
Where in England was Dame Judi Dench born?                              & York \\
From which country did Angola achieve independence in 1975?             & Portugal \\
Which city does David Soul come from?                                   & Chicago \\
Who won Super Bowl XX?                                                  & Chicago Bears \\
\bottomrule
\end{tabular}
\caption{Five OEQA demonstrations, taken verbatim from the TriviaQA prompt
of~\citet{su2024api} (Appendix~C.1). Used by every OEQA template except
\texttt{zero\_shot} (which omits \texttt{\{demos\}}) and \texttt{minimal}
(which omits demos and instructions). The same five demos are reused
across NQ and TriviaQA so that comparisons across templates isolate
template effects rather than demo-content effects.}
\label{tab:oeqa-demos}
\end{table}

\paragraph{Category 1 --- Surface format (3 variants).}
With no fixed label space, the surface-format axis varies the typesetting of
the question/answer \emph{prefix tokens} rather than option labels. The
three variants reuse the anchor prompt body and differ only in
(\texttt{q\_prefix}, \texttt{a\_prefix}): \textsc{Baseline} uses
(\texttt{``Q: ''}, \texttt{``A: ''}); \texttt{lowercase} uses
(\texttt{``q: ''}, \texttt{``a: ''}); and \texttt{full\_prefix} uses
(\texttt{``question: ''}, \texttt{``answer: ''}). The substitution is
applied uniformly to every demonstration line and to the final test-question
line, so the prompt remains internally consistent. The grid is structurally
narrower than its MCQA counterpart ($3$ vs $9$ variants) because OEQA has no
delimiter $\times$ label-set product to sweep.

\paragraph{Category 2 --- Instruction syntax / wording (3 variants).}
As in the MCQA suite, the \textsc{Baseline} body sits at the middle of this
axis. The \texttt{formal} variant keeps the
(\texttt{``Q: ''}, \texttt{``A: ''}) prefixes and the same five demos, but
adds expert-role framing, numbered output guidelines that forbid hedging and
``The answer is\ldots'' preambles, and a
\texttt{``Let's think step by step.''} directive. The \texttt{minimal}
variant is the single line \texttt{\{question\}}, with no demonstrations,
instructions, or Q/A prefixes.

\paragraph{Category 3 --- Content ablation (2 variants).}
\texttt{no\_cot} drops the chain-of-thought directive from the
\textsc{Baseline} body (``Reason step-by-step and answer the following
question with a short factual answer.''~$\rightarrow$~``Answer the following
question with a short factual answer.''); \texttt{zero\_shot} additionally
drops the \texttt{\{demos\}} block. Both retain (\texttt{``Q: ''},
\texttt{``A: ''}) prefixes. There is no \texttt{no\_subject} analogue
because the OEQA anchor has no subject scaffolding to remove.

\section{More Empirical Results for the Effect of Configuration Shift}
\label{app:extra-results}

\subsection{Evaluation Denominators and Diagnostics}
\label{app:evaluation-diagnostics}

The main text aggregates undercoverage and set-size inflation over the audit
grid. Table~\ref{tab:evaluation-diagnostics} records the corresponding
evaluation denominators and auxiliary diagnostics: the number of
configuration summaries, the number of evaluated cell--trial pairs, top-1
base accuracy, unconditional set-size inflation, and, for OEQA,
gold-in-candidate retention.

For OEQA, prediction sets are restricted to the sampled answer pool. If no
sampled candidate matches the gold answer, no conformal threshold can cover
that question. Let $r_q^\xi$ be the minimum score among sampled candidates
that match the gold answer under configuration $\xi$, with
$r_q^\xi=\infty$ if no such candidate exists. For each
calibration--test configuration pair $(\xi_i,\xi_j)$, we evaluate OEQA on
the answerable intersection
\[
\mathcal S_{ij}
=\{q:r_q^{\xi_i}<\infty \ \wedge\ r_q^{\xi_j}<\infty\}.
\]
Calibration and test indices are drawn first and then filtered to
$\mathcal S_{ij}$. Thus OEQA results should be read as \textbf{conditional on this
cell-specific answerable population}. 

This conditioning is needed to separate conformal miscalibration from the
candidate-generation bottleneck. In sampling-based OEQA, driving the
gold-in-candidate rate close to one can require a much larger sampling budget:
\citet{su2024api} note that even 95\%-confidence, 1\%-error probability
estimation would require $9{,}604$ samples. Such budgets are impractical for
our multi-model, multi-dataset configuration audit, so we use a fixed
20-sample budget. At this budget, many failures are simply cases where the
gold answer never appears in the candidate pool; counting them as CP failures
would conflate answer generation with robustness to configuration shift.
The answerable-intersection evaluation therefore isolates the question studied
in this paper: whether CP remains calibrated when the score-generating
configuration changes.
For MCQA, the corresponding retention
rate is $1$ by construction because the fixed four-option label space
contains the gold label.

We also report the unconditional set-size inflation
\[
\mathrm{Infl}_{\mathrm{all}}(\mathcal A)
=
\frac{1}{|\mathcal A|}
\sum_{a=(\xi_0,\xi',r)\in\mathcal A}
\frac{\widehat L_a}{\bar L_{\mathrm{iid}}(\xi')},
\]
computed over all shifted cell--trial pairs, rather than only runs with
$\widehat C\ge 1-\alpha$ as in Section~\ref{sec:empirical}. This separates
the main efficiency conclusion from the valid-run conditioning used by
$\mathrm{Infl}$.

\begin{table*}[t]
	\centering
	\scriptsize
	\begin{tabular}{llrrrrr}
		\toprule
		Axis & Task & Configs & Cell--trials & Gold cand. & Base acc. & $\mathrm{Infl}_{\mathrm{all}}$ \\
		\midrule
		Prompt & MCQA & 72 & $141{,}120/131{,}040$ & $1.00$ & $0.49$ & $1.00$ \\
		Prompt & OEQA & 36 & $17{,}640/15{,}120$ & $0.33$ & $0.60$ & $0.84$ \\
		Temperature & MCQA & 72 & $18{,}000/14{,}400$ & $1.00$ & $0.51$ & $1.01$ \\
		Temperature & OEQA & 36 & $9{,}000/7{,}200$ & $0.31$ & $0.64$ & $0.76$ \\
		Quantization & MCQA & 24 & $11{,}520/8{,}640$ & $1.00$ & $0.65$ & $1.00$ \\
		Quantization & OEQA & 12 & $5{,}760/4{,}320$ & $0.26$ & $0.50$ & $0.95$ \\
		\bottomrule
	\end{tabular}
	\caption{\textbf{Evaluation diagnostics at $1-\alpha=0.9$.}
		Cell--trials are reported as all/shifted cell--trial pairs after
		score filtering. Gold cand. is the gold-in-candidate retention rate;
		for OEQA this is the answerable-intersection retention
		$|\mathcal S_{ij}|/|\mathcal Q|$, while for MCQA it is $1$ by
		construction. Base acc. is the top-1 accuracy of the underlying
		pipeline before conformal filtering, averaged with the same
		configuration and score weighting as the audit grid.}
	\label{tab:evaluation-diagnostics}
\end{table*}

\begin{figure}[t]
	\centering
	\includegraphics[width=\linewidth]{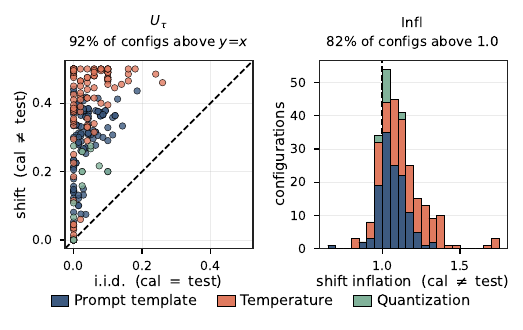}
	\caption{\textbf{Per-configuration coverage damage under
		configuration shift, pooled across task types.} Each summary is for one
		$(\text{dataset},\text{model},\text{score})$ configuration and one shift
		source, pooling over the audit grid and $10$ trials at $1-\alpha=0.9$;
		MCQA and open-ended QA configurations are pooled and colour denotes the
		shift source. \emph{Left (under-coverage):} a configuration's
		in-distribution value (cal\,$=$\,test, $x$) against its shifted value
		(cal\,$\neq$\,test, $y$); the dashed line is $y=x$, the configuration's
		own i.i.d.\ baseline, and a point above it means shift worsens
		under-coverage. \emph{Right (set-size inflation):} the distribution of
		shifted inflation across configurations; inflation is normalized per
		configuration, so each configuration's own i.i.d.\ baseline is exactly
		$1.0$ (dashed line). The percentage above each panel is the share of
		configurations in the worse region (above $y=x$ on the left, above
		$1.0$ on the right).}
	\label{fig:diag-shift}
\end{figure}

\subsection{Per-Configuration Effects of Configuration Shift}
Section~\ref{sec:empirical} reports the configuration-shift effect as an
aggregate over the audit grid; Figure~\ref{fig:diag-shift} shows the same
effect at the level of individual configuration summaries, pooling MCQA and
OEQA. In the undercoverage panel, each point compares a
configuration's in-distribution value with its shifted value for one shift
source, and the dashed $y=x$ line is each point's own i.i.d.\ baseline, so
points above the diagonal are cases where shift worsens validity. In the
inflation panel, each configuration's i.i.d.\ baseline is exactly $1.0$ by
construction, so we instead show the distribution of shifted inflation, with
the dashed line at $1.0$.

The undercoverage panel supports the main validity claim at a finer
granularity: more than nine in ten configuration--axis summaries move into
the worse half-plane under shift. The aggregate effect is therefore not
carried by a few fragile cells, but appears broadly across datasets, models,
scores, task types, and shift axes.

The set-size panel gives the complementary efficiency picture. Most
configurations (about $82\%$) inflate under shift relative to their own
i.i.d.\ baseline, but the distribution stays tightly concentrated near $1.0$
(median $1.08$), with a minority ($14\%$) contracting below their baseline.
This matches the main-body conclusion: configuration shift primarily
manifests as validity loss, while set sizes among valid runs remain close to
their i.i.d.\ scale.

\begin{figure}[t]
	\centering
	\includegraphics[width=\linewidth]{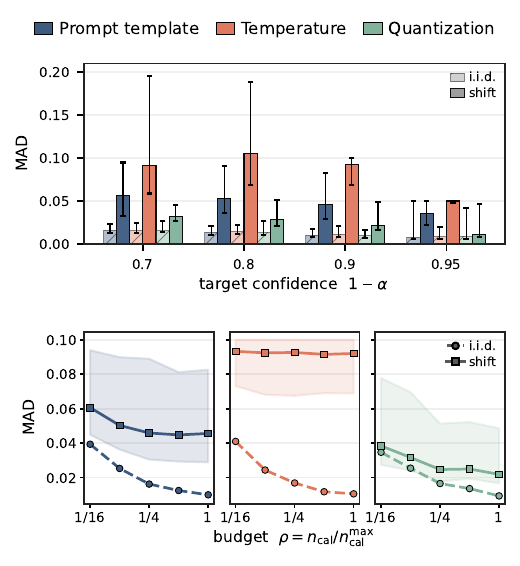}
	\caption{\textbf{Configuration shift increases coverage error across
		targets and calibration budgets.} Median absolute deviation from the
		nominal coverage target, computed from the same audit grid as
		Figure~\ref{fig:results}. Lower is better. \emph{Top:}
		target-confidence sweep at the full calibration budget.
		\emph{Bottom:} calibration-budget sweep at fixed $1-\alpha=0.9$,
		faceted by shift source. Hatched bars and dashed lines denote
		in-distribution cells; solid bars and lines denote shifted cells.
		Because this metric is symmetric around the target, it should be read
		together with the undercoverage results in
		Figures~\ref{fig:results} and~\ref{fig:diag-shift}.}
	\label{fig:mad-combined}
\end{figure}

\subsection{Coverage-Error Magnitude Under Configuration Shift}
The main text uses undercoverage rate as the primary validity metric,
because it directly counts runs that fail to meet the target. As a
complementary severity view, Figure~\ref{fig:mad-combined} reports the
median absolute deviation of empirical coverage from the nominal target.
For a fixed $(\text{dataset},\text{model},\text{score})$ configuration
$g$ and a collection of evaluated cell-trial pairs $\mathcal A$, define
\[
\mathrm{MAD}_{1-\alpha}(g;\mathcal A)
=
\operatorname*{median}_{a\in\mathcal A}
\left|\widehat C_a-(1-\alpha)\right|,
\]
where $\widehat C_a$ is the empirical coverage of cell-trial pair $a$.
We compute this quantity separately for the in-distribution subset
$\mathcal A_{\mathrm{iid}}$ and the shifted subset
$\mathcal A_{\mathrm{shift}}$, using the same partition as
Section~\ref{sec:empirical}.

This metric is intentionally symmetric around the target, so it does not
by itself identify whether errors are conservative or anti-conservative;
the direction is given by the undercoverage analyses above and in
Section~\ref{sec:empirical}. Across the target-confidence sweep, shifted
cells have consistently larger coverage error than their in-distribution
baselines, with temperature again showing the largest deviation. Thus the
main validity result is not only a threshold-crossing effect: under shift,
coverage also moves farther away from the target in magnitude.

The calibration-budget sweep gives the same message from the finite-sample
side. In-distribution MAD decreases as calibration noise is reduced, while
shifted MAD remains separated from the i.i.d.\ baseline, especially under
temperature shift. This supports the main-body interpretation that more
calibration data removes sampling noise but does not remove the systematic
calibration--deployment mismatch induced by configuration shift.

\begin{figure}[t]
	\centering
	\includegraphics[width=\linewidth]{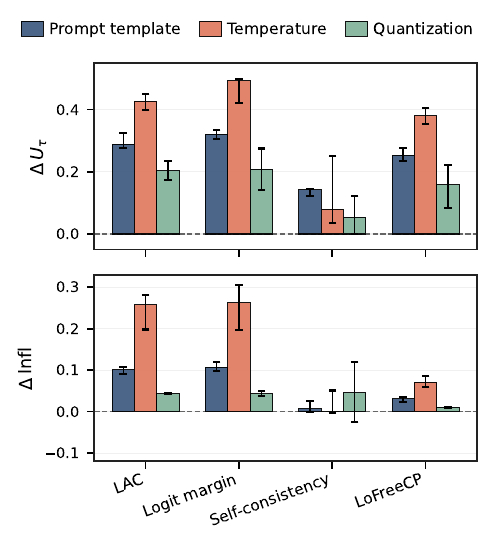}
	\caption{\textbf{Per-score fragility breakdown at $1-\alpha=0.9$.}
		Per-configuration gaps $\Delta U_\tau$ and $\Delta\mathrm{Infl}$
		(defined in Section~\ref{app:per-score})
		are aggregated across configurations per nonconformity score, pooling
		MCQA and OEQA where both apply. White-box scores
		are MCQA-only (logits unavailable on the open-ended pipeline); the
		black-box scores (self-consistency, LoFreeCP)
		pool MCQA and OEQA.
	}
	\label{fig:per-score}
\end{figure}

\subsection{Per-Score Fragility}
\label{app:per-score}
For each nonconformity score we report two per-configuration gaps, the change
in each metric between the shifted and in-distribution evaluation subsets
$\mathcal A_{\mathrm{shift}}$ and $\mathcal A_{\mathrm{iid}}$ (the same
partition as Section~\ref{sec:empirical}):
\[
\Delta U_\tau = U_\tau(\mathcal A_{\mathrm{shift}}) - U_\tau(\mathcal A_{\mathrm{iid}}),
\]
\[
\Delta\mathrm{Infl} = \mathrm{Infl}(\mathcal A_{\mathrm{shift}}) - \mathrm{Infl}(\mathcal A_{\mathrm{iid}}),
\]
the shift-induced change in undercoverage rate $U_\tau$ and set-size inflation
$\mathrm{Infl}$. A larger $\Delta U_\tau$ means shift drives more runs below
the coverage target; $\Delta\mathrm{Infl}$ is signed, with prediction sets
growing when it is positive and shrinking when negative.

Disaggregating by nonconformity score at $1-\alpha=0.9$
(Figure~\ref{fig:per-score}) shows that score choice changes the
degree of fragility, but not the basic failure mode. The two white-box
scores, LAC and logit margin, behave very similarly across all three
shift axes, with temperature producing the largest validity loss for
both. This suggests that, for logits-based CP, configuration sensitivity
is driven more by how the pipeline shifts the predictive distribution
than by the particular white-box score formula.

The black-box scores are more heterogeneous. Self-consistency is the
most stable score in Figure~\ref{fig:per-score}, with small
$\Delta U_\tau$ and little set-size change, whereas LoFreeCP is
substantially more fragile, especially under temperature shift.
Efficiency effects are weaker than validity effects: large
$\Delta\mathrm{Infl}$ appears mainly for white-box scores under
temperature, while black-box set sizes remain close to their
per-configuration i.i.d.\ baselines. Thus the per-score breakdown
reinforces the main conclusion: configuration shift primarily causes
undercoverage, and robustness depends on the score, but no score removes
the need to audit the deployed configuration.

\subsection{Per-configuration Analysis}
\label{app:heatmaps}
Figures~\ref{fig:heat-prompt} and~\ref{fig:heat-ordered} resolve the
per-configuration damage of Figure~\ref{fig:diag-shift} into the full
(calibration, deployment) plane, pooling over model, score, and trial. The
diagonal (cal $=$ test) under-coverage sits at the i.i.d.\ floor on every axis
and both tasks ($0.03$--$0.05$); off the diagonal the mean jumps at once, to
$0.28$/$0.26$ for prompt, $0.42$/$0.30$ for temperature, and $0.19$/$0.14$ for
quantization (MCQA/OEQA).  The aggregate
coverage loss is not driven by a small number of outliers; it is a systematic
effect of calibrating and deploying under different score-generating
configurations.

The heatmaps also show that configuration shift is structured rather than
uniform. For ordered axes, the damage is asymmetric: calibrating under one
configuration and deploying under another can be much worse than reversing the
pair. Quantization gives the clearest example. In both MCQA and open-ended QA,
deployment to \texttt{Q2} is fragile across many calibration precisions
(e.g., \texttt{Q8}/\texttt{Q6}/\texttt{Q4}$\to$\texttt{Q2}), while the reverse
direction, calibrating on \texttt{Q2} and deploying to higher precisions, has
near-zero undercoverage but can produce larger sets. Thus reversing a mismatch
does not remove configuration shift; it often changes the failure mode from
validity loss to conservativeness. Temperature shows a similar directional
structure, with the more harmful direction depending on the task. These paired
reversals show that coverage loss is not explained by configuration distance
alone. The harmful direction is the one that calibrates on configurations with
smaller nonconformity scores and deploys on configurations with larger ones,
supporting the score-distribution view of Section~\ref{sec:bounds}.

Prompt shift shows a different but equally structured pattern. In MCQA, the
dominant source of fragility is not the option delimiter but the answer-label
token itself. Within the surface-format grid, changing only the delimiter while
keeping the label set fixed is comparatively mild, whereas changing the label
set (\texttt{A/B/C}, \texttt{a/b/c}, or \texttt{1/2/3}) produces much larger
undercoverage. The direction is also asymmetric: calibrating on numeric labels
and deploying on letter labels is much more harmful than the reverse. Thus even
a superficial-looking change in the answer interface can shift the
nonconformity distribution enough to invalidate calibration.
OEQA shows a complementary pattern. Because there is no fixed answer
label space, fragility is driven less by prefix typography and more by changes
to the prompt body. Deployment on \texttt{minimal} or \texttt{zero\_shot}
templates is especially fragile, while isolated prefix or chain-of-thought
wording changes are milder.

\section{Proofs and Auxiliary Results for Section~\ref{sec:bounds}}
\label{app:proofs}

This appendix proves the three coverage lower bounds of
Section~\ref{sec:bounds} (i.e., Proposition~\ref{prop:score-tv},
Theorem~\ref{thm:score-ks}, and Theorem~\ref{thm:label-free}) and
collects their matching upper bounds. All four results share one
mechanism: 
after calibration the acceptance region
$E:=(-\infty,\hat\lambda]$ is random but fixed conditional on the
calibration sample; test coverage is $C=\E[P_S^{\xi'}(E)]$, while the
calibration guarantee gives $\E[P_S^{\xi_0}(E)]\ge 1-\alpha$. Thus the
coverage loss is controlled by how much the two laws disagree on $E$.
We isolate this mechanism once
(Appendices~\ref{app:integral-identity}--\ref{app:score-bounds}) and
then specialize it.

\paragraph{Provenance.}
The mechanism we isolate (i.e. bounding coverage loss by a distributional
discrepancy evaluated on the conditionally fixed acceptance
region) is the device underlying robust-validation and
non-exchangeability analyses of CP
\citep{cauchois2024robust,barber2023}, and the calibration-side
envelope it rests on is the standard split-conformal guarantee
\citep{vovk2005,lei2018}. What is specific to this paper is the
\emph{object} to which the template is applied, namely, the score pushforward
under configuration shift, where $P_Z$ is fixed and the score map
moves, together with the half-line (KS) tightening of
Theorem~\ref{thm:score-ks} and the label-free kernel decomposition of
Theorem~\ref{thm:label-free}; the $W_1$ form additionally invokes
Kantorovich--Rubinstein duality \citep{villani2009optimal}.

\subsection{Calibration-side integral identity}
\label{app:integral-identity}

\begin{lemma}[Integral identity]\label{lem:integral-identity}
    For any probability law $\nu$ on $\R$ and any $T_\nu\sim\nu$
    independent of the calibration scores,
    $\Prob(T_\nu\in E)=\E[\nu(E)]$.
\end{lemma}

\begin{proof}
    Conditioning on $S_1,\ldots,S_n$ makes $E$ deterministic while
    $T_\nu$ retains law $\nu$, so
    $\Prob(T_\nu\in E\mid S_1,\ldots,S_n)=\nu(E)$ and the claim
    follows by taking expectations.
\end{proof}

\begin{corollary}[Calibration-law integral bounds]
    \label{cor:calibration-integral}
    Let $S_1,\ldots,S_n\overset{\mathrm{iid}}{\sim}P_S^{\xi_0}$ be the
    calibration scores and $T_P\sim P_S^{\xi_0}$ an independent
    reference. Then $\E[P_S^{\xi_0}(E)]\ge 1-\alpha$. If, in addition,
    $S_1,\ldots,S_n,T_P$ are almost surely distinct, then
    $\E[P_S^{\xi_0}(E)]\le 1-\alpha+\tfrac{1}{n+1}$.
\end{corollary}

\begin{proof}
    The augmented sequence $S_1,\ldots,S_n,T_P$ is i.i.d.\ and hence
    exchangeable. Theorem~\ref{thm:split-cp} applied to this sequence
    gives $\Prob(T_P\in E)\ge 1-\alpha$, and when the scores are
    almost surely distinct the standard split-conformal upper bound
    \citep{lei2018} additionally gives
    $\Prob(T_P\in E)\le 1-\alpha+\tfrac{1}{n+1}$.
    Lemma~\ref{lem:integral-identity} with $\nu=P_S^{\xi_0}$ rewrites
    $\Prob(T_P\in E)$ as $\E[P_S^{\xi_0}(E)]$.
\end{proof}

\subsection{Score-space coverage bounds}
\label{app:score-bounds}

The next lemma is the shared engine: it converts any almost-sure bound
on the disagreement of the two score laws over the acceptance region
into a two-sided bound on the test coverage. Each result in
Section~\ref{sec:bounds} is obtained by supplying such a bound.

\begin{lemma}[Coverage gap from a half-line discrepancy]
    \label{lem:coverage-gap}
    Let $S_1,\ldots,S_n\overset{\mathrm{iid}}{\sim}P_S^{\xi_0}$ be the
    calibration scores and $S_{n+1}\sim P_S^{\xi'}$ independent of
    them. If a constant $\Delta$ satisfies
    $|P_S^{\xi_0}(E)-P_S^{\xi'}(E)|\le\Delta$ almost surely, then
    \begin{equation}\label{eq:gap-lower}
        C\;\ge\;1-\alpha-\Delta.
    \end{equation}
    If, in addition, $S_1,\ldots,S_n$ and an independent reference
    $T_P\sim P_S^{\xi_0}$ are almost surely distinct, then
    \begin{equation}\label{eq:gap-upper}
        C\;\le\;1-\alpha+\tfrac{1}{n+1}+\Delta.
    \end{equation}
\end{lemma}

\begin{proof}
    Since $S_{n+1}\sim P_S^{\xi'}$ is independent of the calibration
    scores, Lemma~\ref{lem:integral-identity} with $\nu=P_S^{\xi'}$
    identifies the test coverage \eqref{eq:deployment-coverage} as
    $C=\Prob(S_{n+1}\in E)=\E[P_S^{\xi'}(E)]$. The almost-sure
    discrepancy bound gives
    \[
        \E[P_S^{\xi_0}(E)]-\Delta
        \;\le\; C \;\le\;
        \E[P_S^{\xi_0}(E)]+\Delta,
    \]
    and Corollary~\ref{cor:calibration-integral} bounds
    $\E[P_S^{\xi_0}(E)]$ between $1-\alpha$ and
    $1-\alpha+\tfrac{1}{n+1}$, the upper side under the no-ties
    condition. Combining the two displays yields \eqref{eq:gap-lower}
    and \eqref{eq:gap-upper}.
\end{proof}

\begin{proof}[Proof of Proposition~\ref{prop:score-tv}]
    By the definition of total variation, every measurable set, in
    particular the realized acceptance region $E$, satisfies
    $|P_S^{\xi_0}(E)-P_S^{\xi'}(E)|\le\dtv(P_S^{\xi_0},P_S^{\xi'})$
    almost surely. Applying the lower bound \eqref{eq:gap-lower} of
    Lemma~\ref{lem:coverage-gap} with
    $\Delta=\dtv(P_S^{\xi_0},P_S^{\xi'})$ gives
    \eqref{eq:score-tv-lower}.
\end{proof}

\begin{proof}[Proof of Theorem~\ref{thm:score-ks}]
	For every realization of $S_1,\ldots,S_n$, if
	$\hat\lambda=+\infty$ then $E=\R$ and
	$|P_S^{\xi_0}(E)-P_S^{\xi'}(E)|=0$. Otherwise
	$\hat\lambda\in\R$, so $E=(-\infty,\hat\lambda]$ and evaluating the
	KS supremum~\eqref{eq:ks-def} at $t=\hat\lambda$ gives
    \begin{equation}\label{eq:score-pathwise}
    \begin{aligned}
        |P_S^{\xi_0}(E)-P_S^{\xi'}(E)|
        &=|F_S^{\xi_0}(\hat\lambda)-F_S^{\xi'}(\hat\lambda)| \\
        &\le\dks(P_S^{\xi_0},P_S^{\xi'}).
    \end{aligned}
    \end{equation}
    Applying the lower bound \eqref{eq:gap-lower} of
    Lemma~\ref{lem:coverage-gap} with
    $\Delta=\dks(P_S^{\xi_0},P_S^{\xi'})$ gives
    \eqref{eq:score-ks-lower}. Since $\dks\le\dtv$, this bound is at
    least as tight as Proposition~\ref{prop:score-tv}; this half-line
    restriction is the step specific to the conformal setting.
\end{proof}

\subsection{Label-free coverage bounds}
\label{app:proof-label-free}

\begin{proof}[Proof of Theorem~\ref{thm:label-free}]
    \emph{TV statement.} For a Borel set $A\subseteq\R$, let
    $f_A(u):=K(A\mid u)\in[0,1]$.
    Assumption~\ref{assump:kernel-stability} gives
    $P_S^{\xi_0}(A)-P_S^{\xi'}(A)=\int f_A\,d(P_U^{\xi_0}-P_U^{\xi'})$.
    Since $f_A$ takes values in $[0,1]$, the dual characterization of
    total variation,
    $\dtv(P,Q)=\sup_{0\le f\le1}\bigl|\int f\,d(P-Q)\bigr|$, yields
    $|P_S^{\xi_0}(A)-P_S^{\xi'}(A)|\le\dtv(P_U^{\xi_0},P_U^{\xi'})$, and
    taking the supremum over $A$ gives the data-processing inequality
    \begin{equation}\label{eq:dp-tv}
        \dtv(P_S^{\xi_0},P_S^{\xi'})
        \le\dtv(P_U^{\xi_0},P_U^{\xi'}).
    \end{equation}
    Substituting \eqref{eq:dp-tv} into Proposition~\ref{prop:score-tv}
    yields \eqref{eq:label-free-tv}.

    \emph{Wasserstein statement.} For probability measures $P,Q$ on
    $(\U,\rho)$ with finite first moments, let $\Pi(P,Q)$ be the set
    of couplings of $P$ and $Q$, and define
    \[
        \Wone^\rho(P,Q)
        :=\inf_{\gamma\in\Pi(P,Q)}
        \mathbb{E}_{(U,V)\sim\gamma}[\rho(U,V)].
    \]
    Kantorovich--Rubinstein duality on Polish metric spaces
    \citep[][Theorem~5.10, particular case ~5.16]{villani2009optimal} gives
    \[
        \Wone^\rho(P,Q)
        =
        \sup_{\|h\|_{\mathrm{Lip}(\rho)}\le 1}
        \Big|\int h\,dP-\int h\,dQ\Big|.
    \]
    For each $t\in\R$, let $g_t(u):=K((-\infty,t]\mid u)$. By
    Assumption~\ref{assump:kernel-stability},
    $F_S^\xi(t)=\int g_t\,dP_U^\xi$ for
    $\xi\in\{\xi_0,\xi'\}$. Since $g_t$ is $L$-Lipschitz with
    respect to $\rho$ uniformly in $t$, the dual representation gives
    \[
    \begin{aligned}
        |F_S^{\xi_0}(t)-F_S^{\xi'}(t)|
        &=\Big|\!\int g_t\,d(P_U^{\xi_0}-P_U^{\xi'})\Big| \\
        &\le L\,\Wone^{\rho}(P_U^{\xi_0},P_U^{\xi'}).
    \end{aligned}
    \]
    Taking the supremum over $t$ yields
    \begin{equation}\label{eq:dp-w1}
        \dks(P_S^{\xi_0},P_S^{\xi'})
        \le L\,\Wone^{\rho}(P_U^{\xi_0},P_U^{\xi'}).
    \end{equation}
    Substituting \eqref{eq:dp-w1} into Theorem~\ref{thm:score-ks}
    gives \eqref{eq:label-free-w1}.
\end{proof}

\subsection{Matching upper bounds}
\label{app:upper-bounds}

The lower-bound proofs above each established an almost-sure bound on
the half-line discrepancy $|P_S^{\xi_0}(E)-P_S^{\xi'}(E)|$. Feeding the
same bound into the upper half of Lemma~\ref{lem:coverage-gap} yields
the matching upper bounds with no extra work.

\begin{remark}[Matching upper bounds]
    \label{rem:upper-bounds}
    Assume the hypotheses of the corresponding lower bound and, in
    addition, that the calibration scores and an independent reference
    $T_P\sim P_S^{\xi_0}$ are almost surely distinct. Then each
    discrepancy
    \[
    \begin{aligned}
        \Delta\in\bigl\{&
        \dtv(P_S^{\xi_0},P_S^{\xi'}),\\
        &\dks(P_S^{\xi_0},P_S^{\xi'}),\\
        &\dtv(P_U^{\xi_0},P_U^{\xi'}),\\
        &L\,\Wone^{\rho}(P_U^{\xi_0},P_U^{\xi'})
        \bigr\}
    \end{aligned}
    \]
    also satisfies
    \begin{equation}\label{eq:matching-upper}
        C\;\le\;1-\alpha+\tfrac{1}{n+1}+\Delta.
    \end{equation}
\end{remark}

\begin{proof}
    For the two score-space discrepancies the half-line bound
    $|P_S^{\xi_0}(E)-P_S^{\xi'}(E)|\le\Delta$ holds almost surely by
    the definition of total variation and by \eqref{eq:score-pathwise},
    respectively. For the two label-free discrepancies the
    data-processing inequalities \eqref{eq:dp-tv} and \eqref{eq:dp-w1}
    give
    $|P_S^{\xi_0}(E)-P_S^{\xi'}(E)|\le\dtv(P_S^{\xi_0},P_S^{\xi'})\le
    \dtv(P_U^{\xi_0},P_U^{\xi'})$ and
    $|P_S^{\xi_0}(E)-P_S^{\xi'}(E)|\le\dks(P_S^{\xi_0},P_S^{\xi'})\le
    L\,\Wone^{\rho}(P_U^{\xi_0},P_U^{\xi'})$.
    In every case the upper bound \eqref{eq:gap-upper} of
    Lemma~\ref{lem:coverage-gap} gives \eqref{eq:matching-upper}.
\end{proof}

\section{Finite-Sample Estimation Details for Bound Diagnostics}
\label{app:bound-verification}
This appendix describes the finite-sample plug-in diagnostics used in
Table~\ref{tab:bounds-empirical}. These estimates are empirical proxies for
the population discrepancies in Section~\ref{sec:bounds}.

\paragraph{Predictor state for the label-free bounds.}
The label-free bounds of Theorem~\ref{thm:label-free} compare the
predictor-state laws $P_U^{\xi_0}$ and $P_U^{\xi'}$ on a common space
$(\U,\rho)$, so the predictor state $U$ must take the same form for every
question. 
For MCQA this is automatic because the candidate set is fixed:
for white-box scores we use
$U=\hat p_\xi(\cdot\mid X)$, and for sampling-based black-box scores
we use $U=F^\xi$; both are four-dimensional probability vectors over
$\Y$.
For OEQA the
candidate set is the per-question sample pool, so $F^\xi$ has a
question-dependent length and its coordinates index answer strings that
change from one question to the next; a discrepancy
$\dtv(P_U^{\xi_0},P_U^{\xi'})$ or $\Wone^{\rho}(P_U^{\xi_0},P_U^{\xi'})$
between two such vectors is not even well defined. 
We therefore summarise each question by its \emph{sorted top-$K'$ frequency
	profile}:
\[
U =
\bigl(F_{(1)}^\xi,\ldots,F_{(K')}^\xi\bigr),
\qquad
F_{(1)}^\xi \ge F_{(2)}^\xi \ge \cdots,
\]
where missing coordinates are zero-padded. We set $K'=4$. 
Furthermore, these quantities for OEQA are computed on the same cell-specific
answerable-intersection population used for coverage evaluation
(Appendix~\ref{app:evaluation-diagnostics}).

\paragraph{Estimating the discrepancies.}
Each discrepancy is a plug-in estimate formed from the calibration and test
draws of a shift cell.
\begin{itemize}[leftmargin=*,nosep]
  \item \emph{Score-space TV} ($\dtv(P_S^{\xi_0},P_S^{\xi'})$). The
  calibration and test scores are histogrammed on a common grid spanning
  their pooled range, and the estimate is the half-$\ell_1$ distance
  $\tfrac12\sum_b|p_b-q_b|$ between the two normalised histograms; the grid
  uses $B$ bins, $B=50$ in our case.
  \item \emph{Score-space KS} ($\dks(P_S^{\xi_0},P_S^{\xi'})$). The two-sample
  Kolmogorov--Smirnov statistic
  $\sup_t|\widehat F_S^{\xi_0}(t)-\widehat F_S^{\xi'}(t)|$ between the
  calibration and test score samples, which has no free parameter and is the
  quantity Theorem~\ref{thm:score-ks} bounds.
  \item \emph{Predictor-state TV} ($\dtv(P_U^{\xi_0},P_U^{\xi'})$). Each
  predictor state is quantised to a regular grid of spacing $1/M$ on the
  simplex (each coordinate rounded to a multiple of $1/M$, with the residual
  assigned to the coordinates of largest fractional part), and the estimate
  is the half-$\ell_1$ distance between the resulting empirical distributions
  over grid cells, with $M=4$.
  \item \emph{Predictor-state $W_1$}
  ($\Wone^{\rho}(P_U^{\xi_0},P_U^{\xi'})$). We use a sliced-$W_1$
  proxy with $J=50$ random unit directions, project both predictor-state
  samples onto each direction, average the resulting one-dimensional
  Wasserstein-$1$ distances, and report
  $\widehat\Delta=\widehat{SW}_1$, corresponding to the heuristic scaling
  $L=1$. 
\end{itemize}

\section{Mitigation Method Details and Ablations}
\label{app:mitigation-detail}

This appendix specifies the six mitigations used in
Section~\ref{sec:mitigations}. 
Throughout this appendix, $U^\xi(X)$ denotes the predictor state used by the
corresponding score family: for white-box MCQA scores this is the predictive
probability vector, while for sampling-based black-box scores it is the
sample-frequency vector.

\subsection{Reweight}
\label{app:mit-reweight}

Reweight keeps the original calibration scores
$S_i=V((X_i,Y_i),\xi_0)$ but changes their mass in the conformal
quantile. We first form predictor states
$u_i=U^{\xi_0}(X_i)$ for calibration examples and
$u'_j=U^{\xi'}(X'_j)$ for the unlabeled deployment
batch. A logistic discriminator
$\hat\pi:\U\to[0,1]$ is trained to distinguish the calibration states
(class $0$) from the deployment states (class $1$). Each calibration
score receives importance weight
$w_i=\hat\pi(u_i)/(1-\hat\pi(u_i))$, clipped to $[0.1,10]$, and
$\hat\lambda$ is the weighted empirical $(1-\alpha)$-quantile of
$\{S_i\}_{i=1}^n$ under these weights.

This is a predictor-state reweighting baseline inspired by weighted split
CP \citep{tibshirani2019covariate},
where a density ratio $dQ/dP$ replaces the uniform calibration
weights. The difference is the object on which the ratio is estimated:
we use $U^\xi(X)$ rather than raw $X$, matching the
label-free predictor-state reduction in
Assumption~\ref{assump:kernel-stability}. The method therefore tests
whether the standard covariate-shift correction can help even though
configuration shift changes the score map $V(\cdot,\xi)$.

\subsection{Mos-U}
\label{app:mit-mosu}

Mos-U builds a mosaic calibration sample by evaluating calibration
examples under alternative configurations. 
For each shift axis, \(\Xi_{\mathrm{aug}}\) denotes the finite configuration
grid used in the mitigation experiment.
For each calibration index
$i$, draw $\xi_i$ uniformly from
$\Xi_{\mathrm{aug}}\setminus\{\xi'\}$ and replace the anchor score by
$V(Z_i,\xi_i)$. The deployment configuration remains fixed at $\xi'$;
only the calibration scores used to compute $\hat\lambda$ are mixed.
The threshold is the ordinary empirical $(1-\alpha)$-quantile of the
mosaic scores $\{V(Z_i,\xi_i)\}_{i=1}^n$.

The closest analogy is prompt ensembling, including uniform prompt
ensembles and Bayesian prompt ensembles
\citep{jiang2023cape,tonolini2024bayespe}. Those methods combine
prompt-conditioned predictions or probabilities to improve predictive
stability. Mos-U transfers the same pooling idea to conformal
calibration: the pooled object is the calibration score distribution used to set the
conformal threshold, not the deployment prediction itself.

\subsection{\texorpdfstring{$\alpha$}{alpha}-Inf}
\label{app:mit-ainf}

\mbox{$\alpha$-Inf} turns the Wasserstein term in the label-free bound into a conservative
plug-in correction. It estimates
\[
    \widehat\varepsilon
    =
    L\cdot\widehat\Wone\!\bigl(\hat P_U^{\xi_0},\hat P_U^{\xi'}\bigr)
\]
by a sliced-\(W_1\) proxy using \(50\) random one-dimensional projections of
the predictor-state vectors, on the
empirical predictor distributions
$\hat P_U^{\xi_0}=\tfrac{1}{n}\sum_{i=1}^n\delta_{U^{\xi_0}(X_i)}$
and
$\hat P_U^{\xi'}=\tfrac{1}{m}\sum_{j=1}^m\delta_{U^{\xi'}(X'_j)}$,
with $L=1$. As in Appendix~\ref{app:bound-verification}, this is a heuristic sliced-\(W_1\)
proxy rather than a finite-sample certificate. 
The miscoverage budget is then deflated to
$\alpha'=\max(\alpha-\widehat\varepsilon,0)$, and the conformal
threshold is recomputed as the $(1-\alpha')$-quantile of the original
calibration scores.

This follows the robust-CP idea of paying for distributional mismatch
by enlarging the threshold. \citet{cauchois2024robust} use
worst-case quantiles over an $f$-divergence ball, while
\citet{xu2025wrcp} use Wasserstein bounds and training-time
regularization. Here the correction is label-free and inference-only:
the discrepancy in Theorem~\ref{thm:label-free} is computed on
predictor states, then charged directly against $\alpha$.

\subsection{Recal}
\label{app:mit-recal}

Recal uses the same $\widehat\varepsilon$ as \mbox{$\alpha$-Inf} only
as a gate. For each
$(\text{model},\text{dataset},\text{score family})$, we estimate a
noise floor $\beta$ as the $95$th percentile of
$\widehat\varepsilon$ on i.i.d.\ cells with $\xi'=\xi_0$. If
$\widehat\varepsilon\le\beta$, Recal keeps the original conformal
threshold. If $\widehat\varepsilon>\beta$, it spends a labeled
deployment budget $k$ and recomputes
$\hat\lambda$ as the $(1-\alpha)$-quantile of
$\{V((X_l,Y_l),\xi')\}_{l=1}^k$. 

Recal is related to adaptive conformal inference
\citep{gibbs2021adaptive} in that both use deployment-side evidence to
modify conformal thresholds. The deployment evidence differs: adaptive
conformal inference updates online from realized coverage errors,
whereas Recal first detects a static configuration change from
unlabeled predictor states and only then asks for a small labeled
deployment sample. This design matches LLM configuration changes such
as a new prompt template, temperature, or quantization level, which are
usually discrete changes rather than gradual time-series drift.

We set $k=50$ by an expected-tail-count rule: at our main target
$1-\alpha=0.9$, the upper tail has mass $\alpha=0.1$, so requiring about
five labeled tail examples gives $k\alpha\approx5$. This keeps the
recalibrated threshold from being driven by one or two extremes while
remaining a small labeled deployment budget. The prompt-axis ablation in
Figure~\ref{fig:recal-budget} supports this choice: improvements
mostly saturate after $\sim\!20$ labels, undercoverage falls from
$U_\tau=0.32$ to $\approx\!0.06$ and stays flat across $k\in[20,60]$,
whereas $k=10$ remains noisy ($U_\tau=0.22$).

\begin{figure}[t]
	\centering
	\includegraphics[width=\linewidth]{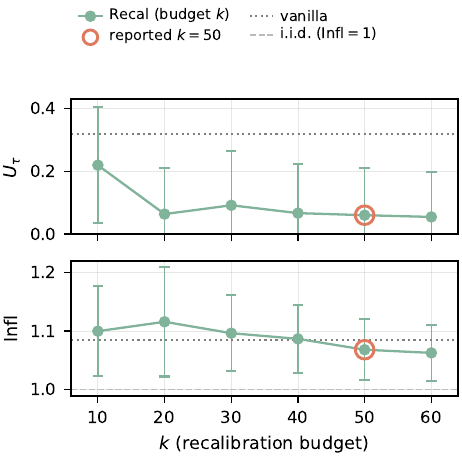}
	\caption{\textbf{Effect of the Recal labeled budget $k$ (prompt axis).}
		Median $U_\tau$ (top) and $\mathrm{Infl}$ (bottom) at
		$1-\alpha=0.9$, pooled over
		$(\text{model},\text{dataset},\text{score})$. Error bars are
		cross-configuration standard deviations; dotted/dashed lines are
		vanilla CP and i.i.d.\ $\mathrm{Infl}=1$; the open ring marks
		$k=50$.}
	\label{fig:recal-budget}
\end{figure}

\begin{figure}[t]
	\centering
	\includegraphics[width=\linewidth]{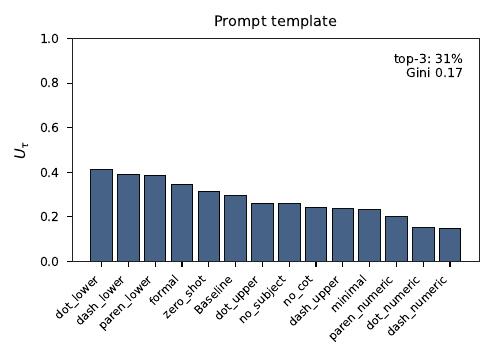}
	\caption{\textbf{The fragility profile, illustrated on the prompt
			axis.} Each bar is the deployment-side undercoverage rate
		$U_\tau$ of vanilla CP for one MCQA prompt template, pooled over
		$(\text{model},\text{score},\text{trial})$ and sorted from most to
		least fragile.}
	\label{fig:fragility-profile}
\end{figure}

\subsection{Mos-F}
\label{app:mit-mosf}
Mos-F is Mos-U with fragility-weighted sampling. For each shift axis,
we estimate an offline fragility profile
$\{f_\xi\}_\xi$, where
$f_\xi=U_\tau$ is the empirical undercoverage rate of
vanilla CP when $\xi$ is used as the deployment configuration, pooled over the same configuration grid as Section~\ref{sec:empirical}
but estimated on auxiliary audit splits disjoint from the final
mitigation evaluation. Calibration
configuration $\xi_i$ is then drawn from
$\Xi_{\mathrm{aug}}\setminus\{\xi'\}$ with probability proportional to
$f_\xi$, and the calibration score is $V(Z_i,\xi_i)$. Thus the mosaic
over-samples configurations that most often break vanilla coverage.
Figure~\ref{fig:fragility-profile} illustrates the profile on the
prompt axis.

Mos-F is closest to weighted prompt ensembling, especially BayesPE
\citep{tonolini2024bayespe}, but its weights serve a different
objective. BayesPE weights prompts by evidence for prediction quality;
Mos-F weights configurations by the empirical undercoverage rate.

\subsection{Anc}
\label{app:mit-anck3}

Anc keeps the mosaic construction but uses a small, fixed calibration
pool. The anchor or reference configuration is the original calibration
configuration \(\xi_0\). Let $\mathcal T_K$ be the $K$ non-anchor configurations with the
largest fragility weights $f_\xi$ on the relevant shift axis. Anc draws
each calibration configuration uniformly from
$\{\xi_0\}\cup\mathcal T_K$ and uses the corresponding score
$V(Z_i,\xi_i)$. We report $K=3$, so the calibration pool always
contains the anchor configuration plus the three most fragile
non-anchor configurations.

Figure~\ref{fig:anc-k} ablates $K$ on the prompt axis. Undercoverage is lowest at $K=3$, falling from $U_\tau=0.32$
under vanilla CP to $0.07$, and rises again at $K=5$ ($0.17$) and $K=7$
($0.18$): enlarging the pool past the few genuinely fragile templates
dilutes it with easier configurations, which lowers the calibration
quantile and erodes the recovered coverage. Set-size inflation is
essentially flat in $K$ ($\mathrm{Infl}\in[1.04,1.07]$, between the
i.i.d.\ reference $1$ and vanilla's $1.08$), so the coverage gained at
$K=3$ does not increase set size relative to the vanilla shifted baseline in this
prompt-axis ablation. We therefore use \(K=3\) as a simple fixed operating point in the main
mitigation comparison.

\begin{figure}[t]
	\centering
	\includegraphics[width=\linewidth]{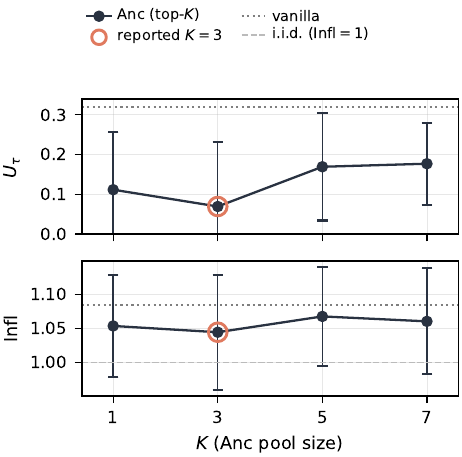}
	\caption{\textbf{Effect of the Anc pool size $K$ (prompt axis).}
		Median $U_\tau$ (top) and $\mathrm{Infl}$ (bottom) at
		$1-\alpha=0.9$, pooled over
		$(\text{model},\text{dataset},\text{score})$. Error bars are
		cross-configuration standard deviations; dotted/dashed lines are
		vanilla CP and i.i.d.\ $\mathrm{Infl}=1$.}
	\label{fig:anc-k}
\end{figure}

\section{Extended Related Work}
\label{app:related-extended}

This appendix expands the related work of Section~\ref{sec:related}. We give
a fuller account of the literatures our paper draws on and connects to, and
state where our contribution sits in each.

\paragraph{Uncertainty quantification for LLMs.}
Uncertainty quantification (UQ) for LLMs has produced a range of confidence
signals that differ mainly in the model access they require
\citep{shorinwa2024uqsurvey}. \emph{Logit-based} methods read uncertainty
off the model's output distribution, spanning sequence likelihood and
predictive entropy, semantic entropy that aggregates probability mass over
meaning-equivalent generations, and the model's own probability that an
answer is correct \citep{kuhn2023semantic,kadavath2022knowwhattheyknow}.
When logits are inaccessible, two
\emph{logit-free} families remain: \emph{verbalized} methods prompt the
model to state its confidence directly
\citep{lin2022teaching,tian2023justask,xiong2024canllms}; and
\emph{frequency-based} methods sample multiple generations and quantify
uncertainty from their agreement or semantic dispersion
\citep{wang2023selfconsistency,manakul2023selfcheckgpt,lin2024generating,robustllm}.
These signals are informative but heuristic---none carries a finite-sample
guarantee that the reported confidence matches the true error rate.
CP is the distribution-free branch of this landscape: it
turns any such signal, logit-based or logit-free, into a nonconformity
score with a coverage guarantee \citep{campos2024conformal}, and the
white-box and black-box scores we study (Section~\ref{sec:cp-llm}) are
exactly conformal wrappers of the logit-based and frequency-based signals
above. That guarantee, however, is implicitly conditional on a fixed
score-generating pipeline. Our work makes this conditioning explicit,
shows it is routinely violated in deployment, and quantifies the resulting
coverage loss.

\paragraph{Applications of CP.}
CP has been applied across a broad range of machine-learning
tasks, including computer vision \citep{angelopoulos2021imagenet}, natural language processing \citep{quach2024clm}, and graph reasoning
\citep{zhu-etal-2025-conformalized,zhu-etal-2025-predicate,zhu-etal-2025-certainty,zhu-2025-thesis,zhu2025uncertainty}.
More recently, CP has been instantiated across the LLM task spectrum. For closed-form
answers it has been applied to multiple-choice and open-ended QA
\citep{kumar2023mcq,su2024api,wang2024conu,wang2025sconu}; for free-form
text, to open-ended and autoregressive generation
\citep{quach2024clm,ravfogel2023nucleus,deutschmann2024beam,ulmer2024noneq,farinhas2024noncrc};
and to higher-level objectives such as factuality, reasoning, and
retrieval-augmented generation
\citep{mohri2024factuality,cherian2024enhanced,rubintoles2025coherent,jiang2025linguistic,li2024traq,feng2025conformalrag},
as well as decision-making and agentic control
\citep{ren2023knowno,si2026ccpo,vishwakarma2025prune}. 
Across this body of work the
inference pipeline (e.g., prompt template, decoding parameters, and model
deployment) is fixed as part of the experimental protocol, and where
alternatives are examined at all they appear as auxiliary sensitivity
checks rather than as a modeled source of shift. Our paper isolates that
pipeline as a first-class variable and shows that varying it erodes the
very guarantees these methods rely on.

\paragraph{CP under distribution shift.}
Extensions of CP beyond exchangeability mostly target shifts in the data
law $P_Z$. One line reweights calibration scores by a known or estimated
covariate-shift likelihood ratio \citep{tibshirani2019covariate}; a second
handles general non-exchangeable or online sequences through fixed or
adaptively updated weights \citep{gibbs2021adaptive,barber2023}; a third
constructs prediction sets robust to a whole neighbourhood of test
distributions
\citep{cauchois2024robust,xu2025wrcp,yang2026multidist,aolaritei2025lp},
and conformalized LLMs have themselves been studied under covariate shift
\citep{hu2026cofact}. All of these keep the score map fixed and absorb a
moving $P_Z$. Configuration shift is the orthogonal case: $P_Z$ is fixed
and the score map $V(\cdot,\xi)$ moves.

\paragraph{Shift detection and two-sample testing.}
Our label-free diagnostics also connect to two-sample methods for detecting
distribution shift from unlabeled deployment data. Kernel two-sample tests
such as MMD \citep{gretton2012mmd} compare calibration and deployment samples
through a distributional discrepancy. In the same spirit, our label-free
bounds compare the predictor-state laws $P_U^{\xi_0}$ and $P_U^{\xi'}$,
using TV or $W_1$ not as an end in itself, but as a quantity that controls
coverage loss under configuration shift. This choice aligns with empirical
evidence that shift detection in machine-learning systems is often more
effective on low-dimensional model outputs than on raw inputs
\citep{rabanser2019failingloudly}. The closest conceptual connection is
\citet{podkopaev2021labelshift}, who obtain distribution-free uncertainty
control under label shift using unlabeled target data and predictor outputs.
Our setting is different: the data law is fixed, while the LLM configuration
and hence the score map changes. Thus our contribution is not a new generic
two-sample test, but a coverage-oriented use of predictor-state discrepancies
for diagnosing configuration shift.

\paragraph{Pipeline sensitivity and ensembling.}
That LLM behaviour is sensitive to prompt formatting, decoding choices, and
quantization is by now well documented
\citep{sclar2024quantifying,shi-etal-2024-thorough,helcig2026statistically},
and a natural response is to ensemble across pipeline variants to stabilise
predictions or probabilities \citep{jiang2023cape,tonolini2024bayespe}.
This prior work targets the stability of point predictions or the
calibration of probabilities. We make two distinct points. First,
prediction stability is not coverage stability: CP validity depends on the
entire nonconformity-score distribution that sets the calibration
threshold, not on whether the top prediction is unchanged. Second, we
repurpose the ensembling idea for the conformal objective---our
\emph{Mos-U}, \emph{Mos-F}, and \emph{Anc} mitigations ensemble over the
calibration configuration and are designed and evaluated directly in terms
of CP validity rather than predictive accuracy.

\begin{figure*}[t]
	\centering
	\includegraphics[width=\textwidth]{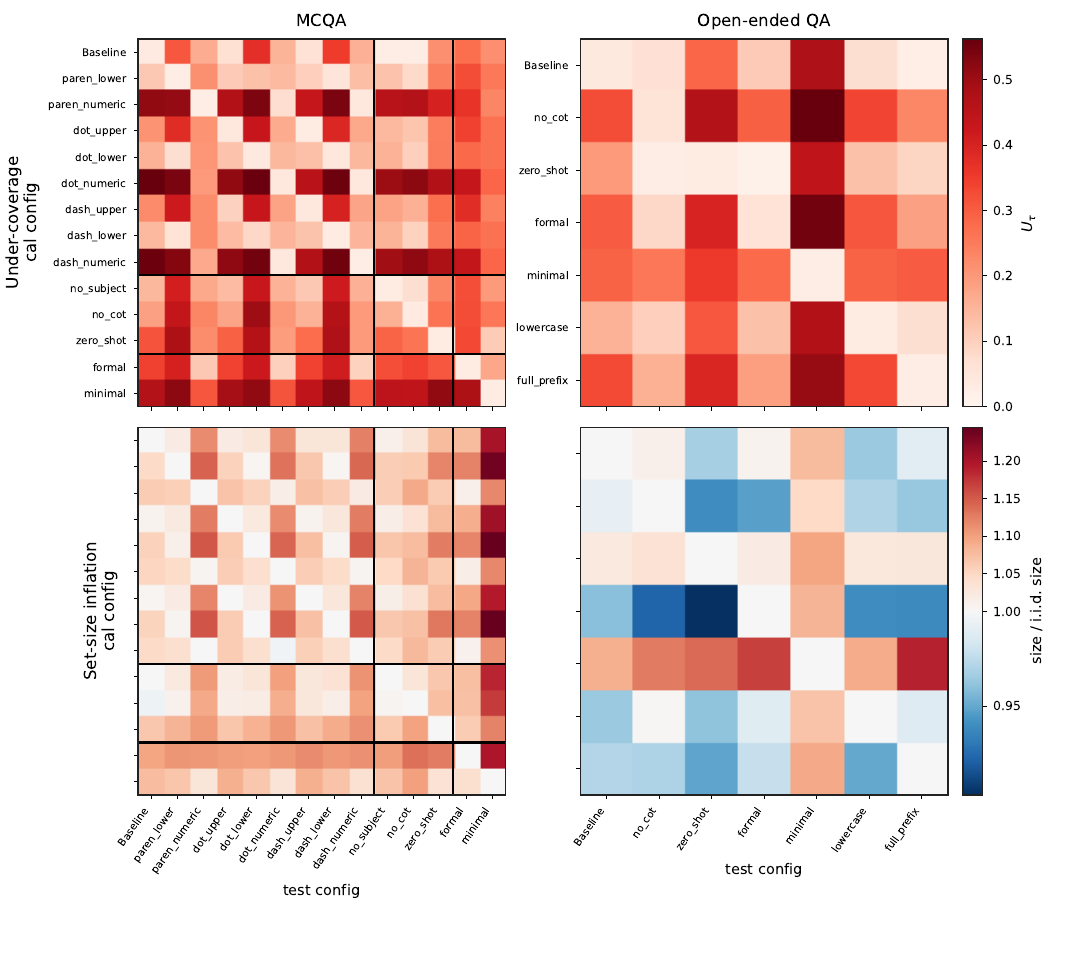}
	\caption{\textbf{Cross-template coverage damage at prompt
			granularity.} Each cell is one (calibration template, test
		template) pair, summarised over (model, score, trial); rows index
		the calibration template, columns the test template.
		\emph{Top row:} under-coverage rate $U_\tau$ (darker red $=$ more
		violation). \emph{Bottom row:} set-size inflation conditioned on
		$c\ge0.9$ (red $=$ inflated, blue $=$ shrunk, white $=$ i.i.d.).
		The diagonal (cal $=$ test) is the i.i.d.\ baseline. \emph{Left:}
		MCQA, $14$ templates ordered format\,/\,content\,/\,wrapper with
		black cluster boundaries; \emph{right:} open-ended QA, $7$
		templates. Off-diagonal cells are systematically darker; in the
		inflation panel MCQA sets grow modestly off-diagonal while
		open-ended QA stays close to its per-configuration i.i.d.\
		baseline (near-white, mild mixed growth and shrinkage).}
	\label{fig:heat-prompt}
\end{figure*}

\begin{figure*}[t]
	\centering
	\includegraphics[width=\textwidth]{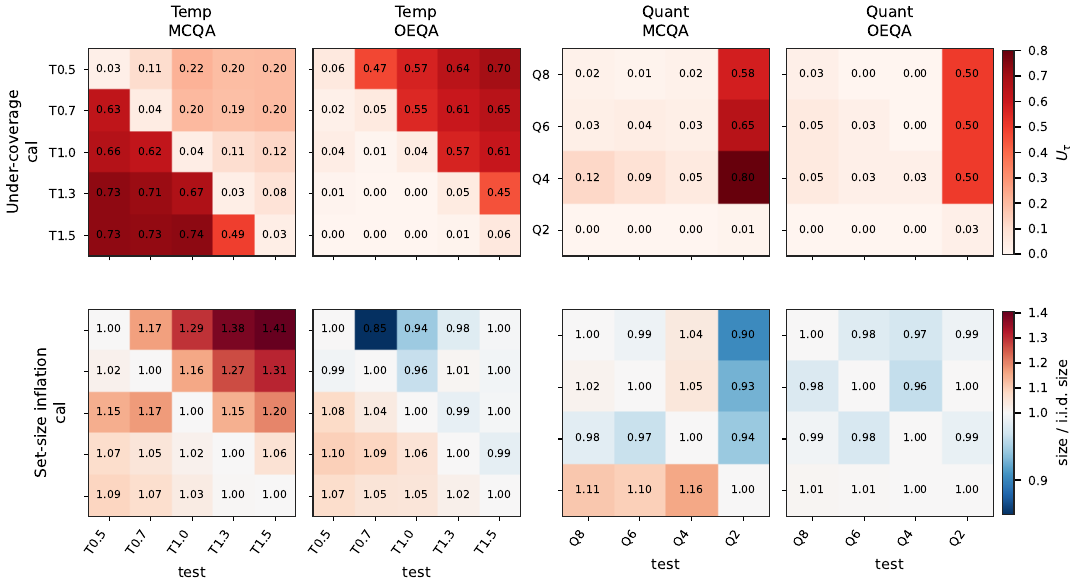}
	\caption{\textbf{Cross-configuration coverage damage on the ordered
			shift axes.} As Figure~\ref{fig:heat-prompt} but for decoding
		temperature ($5$ levels; $T{=}1.0$ is the reference) and weight
		quantization ($4$ GGUF precisions, $3$ instruction models), with cell values
		printed. Top row: under-coverage rate $U_\tau$ (Reds, shared scale
		across both axes); bottom row: set-size inflation (RdBu\_r,
		diverging at $1.0$). Under-coverage is near zero on the i.i.d.\
		diagonal and rises off it asymmetrically---worst when the deployment
		configuration carries larger nonconformity scores than calibration
		(most aggressively under $\texttt{Q2}$), not simply when the two
		sit far apart.}
	\label{fig:heat-ordered}
\end{figure*}

\section{AI Assistants In Writing}
We use AI to enhance our writing skills, abstaining from its use in research and coding endeavors.

\end{document}